\documentclass{article}

 \usepackage[preprint]{neurips_2026}

\usepackage[utf8]{inputenc} % allow utf-8 input
\usepackage[T1]{fontenc}    % use 8-bit T1 fonts
\usepackage{hyperref}       % hyperlinks
\usepackage{url}            % simple URL typesetting
\usepackage{booktabs}       % professional-quality tables
\usepackage{amsfonts}       % blackboard math symbols
\usepackage{nicefrac}       % compact symbols for 1/2, etc.
\usepackage{microtype}      % microtypography
\usepackage{xcolor}         % colors

\usepackage{algorithm}
\usepackage{algorithmic}

\usepackage{amsmath,amsthm}
\newtheorem{theorem}{Theorem}
\newtheorem{lemma}{Lemma}
\newtheorem{corollary}{Corollary}
\newtheorem{proposition}{Proposition}
\newtheorem{problem}{Problem Statement}

\theoremstyle{plain}
\newtheoremstyle{TheoremNum}
    {\topsep}{\topsep}              %%% space between body and thm
    {\itshape}                      %%% Thm body font
    {}                              %%% Indent amount (empty = no indent)
    {\bfseries}                     %%% Thm head font
    {.}                             %%% Punctuation after thm head
    { }                             %%% Space after thm head
    {\thmname{#1}\thmnote{ \bfseries #3}}%%% Thm head spec
\theoremstyle{TheoremNum}
\newtheorem{proposition-repeat}{Proposition}
\newtheorem{lemma-repeat}{Lemma}
\newtheorem{theorem-repeat}{Theorem}

\newcommand{\tr}{\mathrm{tr}}
\newcommand{\mils}{h_\text{max}}

\usepackage[breakable]{tcolorbox}
    \usepackage{parskip} % Stop auto-indenting (to mimic markdown behaviour)

    \usepackage{graphicx}
    \usepackage{float}
    \usepackage{xcolor} % Allow colors to be defined
    \usepackage{enumerate} % Needed for markdown enumerations to work
    \usepackage{geometry} % Used to adjust the document margins
    \usepackage{amsmath} % Equations
    \usepackage{amssymb} % Equations
    \usepackage{textcomp} % defines textquotesingle
    \usepackage{upquote} % Upright quotes for verbatim code
    \usepackage{eurosym} % defines \euro

    \usepackage{iftex}
    \ifPDFTeX
        \usepackage[T1]{fontenc}
        \IfFileExists{alphabeta.sty}{
              \usepackage{alphabeta}
          }{
              \usepackage[mathletters]{ucs}
              \usepackage[utf8x]{inputenc}
          }
    \else
        \usepackage{fontspec}
        \usepackage{unicode-math}
    \fi

    \usepackage{fancyvrb} % verbatim replacement that allows latex
    \usepackage{grffile} % extends the file name processing of package graphics
    \makeatletter % fix for old versions of grffile with XeLaTeX
    \@ifpackagelater{grffile}{2019/11/01}
    {
    }
    {
      \def\Gread@@xetex#1{%
        \IfFileExists{"\Gin@base".bb}%
        {\Gread@eps{\Gin@base.bb}}%
        {\Gread@@xetex@aux#1}%
      }
    }
    \makeatother
    \usepackage[Export]{adjustbox} % Used to constrain images to a maximum size
    \adjustboxset{max size={0.9\linewidth}{0.9\paperheight}}

    \usepackage{hyperref}
    \usepackage{longtable} % longtable support required by pandoc >1.10
    \usepackage{booktabs}  % table support for pandoc > 1.12.2
    \usepackage{array}     % table support for pandoc >= 2.11.3
    \usepackage{calc}      % table minipage width calculation for pandoc >= 2.11.1
    \usepackage[inline]{enumitem} % IRkernel/repr support (it uses the enumerate* environment)
    \usepackage[normalem]{ulem} % ulem is needed to support strikethroughs (\sout)
    \usepackage{mathrsfs}

    \definecolor{urlcolor}{rgb}{0,.145,.698}
    \definecolor{linkcolor}{rgb}{.71,0.21,0.01}
    \definecolor{citecolor}{rgb}{.12,.54,.11}

    \definecolor{ansi-black}{HTML}{3E424D}
    \definecolor{ansi-black-intense}{HTML}{282C36}
    \definecolor{ansi-red}{HTML}{E75C58}
    \definecolor{ansi-red-intense}{HTML}{B22B31}
    \definecolor{ansi-green}{HTML}{00A250}
    \definecolor{ansi-green-intense}{HTML}{007427}
    \definecolor{ansi-yellow}{HTML}{DDB62B}
    \definecolor{ansi-yellow-intense}{HTML}{B27D12}
    \definecolor{ansi-blue}{HTML}{208FFB}
    \definecolor{ansi-blue-intense}{HTML}{0065CA}
    \definecolor{ansi-magenta}{HTML}{D160C4}
    \definecolor{ansi-magenta-intense}{HTML}{A03196}
    \definecolor{ansi-cyan}{HTML}{60C6C8}
    \definecolor{ansi-cyan-intense}{HTML}{258F8F}
    \definecolor{ansi-white}{HTML}{C5C1B4}
    \definecolor{ansi-white-intense}{HTML}{A1A6B2}
    \definecolor{ansi-default-inverse-fg}{HTML}{FFFFFF}
    \definecolor{ansi-default-inverse-bg}{HTML}{000000}

    \definecolor{outerrorbackground}{HTML}{FFDFDF}

    \DefineVerbatimEnvironment{Highlighting}{Verbatim}{commandchars=\\\{\}}
    \let\Oldtex\TeX
    \let\Oldlatex\LaTeX
    \renewcommand{\TeX}{\textrm{\Oldtex}}
    \renewcommand{\LaTeX}{\textrm{\Oldlatex}}
    \title{Numerical Simulations}

\makeatletter
\def\PY@reset{\let\PY@it=\relax \let\PY@bf=\relax%
    \let\PY@ul=\relax \let\PY@tc=\relax%
    \let\PY@bc=\relax \let\PY@ff=\relax}
\def\PY@tok#1{\csname PY@tok@#1\endcsname}
\def\PY@toks#1+{\ifx\relax#1\empty\else%
    \PY@tok{#1}\expandafter\PY@toks\fi}
\def\PY@do#1{\PY@bc{\PY@tc{\PY@ul{%
    \PY@it{\PY@bf{\PY@ff{#1}}}}}}}
\def\PY#1#2{\PY@reset\PY@toks#1+\relax+\PY@do{#2}}

\@namedef{PY@tok@w}{\def\PY@tc##1{\textcolor[rgb]{0.73,0.73,0.73}{##1}}}
\@namedef{PY@tok@c}{\let\PY@it=\textit\def\PY@tc##1{\textcolor[rgb]{0.24,0.48,0.48}{##1}}}
\@namedef{PY@tok@cp}{\def\PY@tc##1{\textcolor[rgb]{0.61,0.40,0.00}{##1}}}
\@namedef{PY@tok@k}{\let\PY@bf=\textbf\def\PY@tc##1{\textcolor[rgb]{0.00,0.50,0.00}{##1}}}
\@namedef{PY@tok@kp}{\def\PY@tc##1{\textcolor[rgb]{0.00,0.50,0.00}{##1}}}
\@namedef{PY@tok@kt}{\def\PY@tc##1{\textcolor[rgb]{0.69,0.00,0.25}{##1}}}
\@namedef{PY@tok@o}{\def\PY@tc##1{\textcolor[rgb]{0.40,0.40,0.40}{##1}}}
\@namedef{PY@tok@ow}{\let\PY@bf=\textbf\def\PY@tc##1{\textcolor[rgb]{0.67,0.13,1.00}{##1}}}
\@namedef{PY@tok@nb}{\def\PY@tc##1{\textcolor[rgb]{0.00,0.50,0.00}{##1}}}
\@namedef{PY@tok@nf}{\def\PY@tc##1{\textcolor[rgb]{0.00,0.00,1.00}{##1}}}
\@namedef{PY@tok@nc}{\let\PY@bf=\textbf\def\PY@tc##1{\textcolor[rgb]{0.00,0.00,1.00}{##1}}}
\@namedef{PY@tok@nn}{\let\PY@bf=\textbf\def\PY@tc##1{\textcolor[rgb]{0.00,0.00,1.00}{##1}}}
\@namedef{PY@tok@ne}{\let\PY@bf=\textbf\def\PY@tc##1{\textcolor[rgb]{0.80,0.25,0.22}{##1}}}
\@namedef{PY@tok@nv}{\def\PY@tc##1{\textcolor[rgb]{0.10,0.09,0.49}{##1}}}
\@namedef{PY@tok@no}{\def\PY@tc##1{\textcolor[rgb]{0.53,0.00,0.00}{##1}}}
\@namedef{PY@tok@nl}{\def\PY@tc##1{\textcolor[rgb]{0.46,0.46,0.00}{##1}}}
\@namedef{PY@tok@ni}{\let\PY@bf=\textbf\def\PY@tc##1{\textcolor[rgb]{0.44,0.44,0.44}{##1}}}
\@namedef{PY@tok@na}{\def\PY@tc##1{\textcolor[rgb]{0.41,0.47,0.13}{##1}}}
\@namedef{PY@tok@nt}{\let\PY@bf=\textbf\def\PY@tc##1{\textcolor[rgb]{0.00,0.50,0.00}{##1}}}
\@namedef{PY@tok@nd}{\def\PY@tc##1{\textcolor[rgb]{0.67,0.13,1.00}{##1}}}
\@namedef{PY@tok@s}{\def\PY@tc##1{\textcolor[rgb]{0.73,0.13,0.13}{##1}}}
\@namedef{PY@tok@sd}{\let\PY@it=\textit\def\PY@tc##1{\textcolor[rgb]{0.73,0.13,0.13}{##1}}}
\@namedef{PY@tok@si}{\let\PY@bf=\textbf\def\PY@tc##1{\textcolor[rgb]{0.64,0.35,0.47}{##1}}}
\@namedef{PY@tok@se}{\let\PY@bf=\textbf\def\PY@tc##1{\textcolor[rgb]{0.67,0.36,0.12}{##1}}}
\@namedef{PY@tok@sr}{\def\PY@tc##1{\textcolor[rgb]{0.64,0.35,0.47}{##1}}}
\@namedef{PY@tok@ss}{\def\PY@tc##1{\textcolor[rgb]{0.10,0.09,0.49}{##1}}}
\@namedef{PY@tok@sx}{\def\PY@tc##1{\textcolor[rgb]{0.00,0.50,0.00}{##1}}}
\@namedef{PY@tok@m}{\def\PY@tc##1{\textcolor[rgb]{0.40,0.40,0.40}{##1}}}
\@namedef{PY@tok@gh}{\let\PY@bf=\textbf\def\PY@tc##1{\textcolor[rgb]{0.00,0.00,0.50}{##1}}}
\@namedef{PY@tok@gu}{\let\PY@bf=\textbf\def\PY@tc##1{\textcolor[rgb]{0.50,0.00,0.50}{##1}}}
\@namedef{PY@tok@gd}{\def\PY@tc##1{\textcolor[rgb]{0.63,0.00,0.00}{##1}}}
\@namedef{PY@tok@gi}{\def\PY@tc##1{\textcolor[rgb]{0.00,0.52,0.00}{##1}}}
\@namedef{PY@tok@gr}{\def\PY@tc##1{\textcolor[rgb]{0.89,0.00,0.00}{##1}}}
\@namedef{PY@tok@ge}{\let\PY@it=\textit}
\@namedef{PY@tok@gs}{\let\PY@bf=\textbf}
\@namedef{PY@tok@gp}{\let\PY@bf=\textbf\def\PY@tc##1{\textcolor[rgb]{0.00,0.00,0.50}{##1}}}
\@namedef{PY@tok@go}{\def\PY@tc##1{\textcolor[rgb]{0.44,0.44,0.44}{##1}}}
\@namedef{PY@tok@gt}{\def\PY@tc##1{\textcolor[rgb]{0.00,0.27,0.87}{##1}}}
\@namedef{PY@tok@err}{\def\PY@bc##1{{\setlength{\fboxsep}{\string -\fboxrule}\fcolorbox[rgb]{1.00,0.00,0.00}{1,1,1}{\strut ##1}}}}
\@namedef{PY@tok@kc}{\let\PY@bf=\textbf\def\PY@tc##1{\textcolor[rgb]{0.00,0.50,0.00}{##1}}}
\@namedef{PY@tok@kd}{\let\PY@bf=\textbf\def\PY@tc##1{\textcolor[rgb]{0.00,0.50,0.00}{##1}}}
\@namedef{PY@tok@kn}{\let\PY@bf=\textbf\def\PY@tc##1{\textcolor[rgb]{0.00,0.50,0.00}{##1}}}
\@namedef{PY@tok@kr}{\let\PY@bf=\textbf\def\PY@tc##1{\textcolor[rgb]{0.00,0.50,0.00}{##1}}}
\@namedef{PY@tok@bp}{\def\PY@tc##1{\textcolor[rgb]{0.00,0.50,0.00}{##1}}}
\@namedef{PY@tok@fm}{\def\PY@tc##1{\textcolor[rgb]{0.00,0.00,1.00}{##1}}}
\@namedef{PY@tok@vc}{\def\PY@tc##1{\textcolor[rgb]{0.10,0.09,0.49}{##1}}}
\@namedef{PY@tok@vg}{\def\PY@tc##1{\textcolor[rgb]{0.10,0.09,0.49}{##1}}}
\@namedef{PY@tok@vi}{\def\PY@tc##1{\textcolor[rgb]{0.10,0.09,0.49}{##1}}}
\@namedef{PY@tok@vm}{\def\PY@tc##1{\textcolor[rgb]{0.10,0.09,0.49}{##1}}}
\@namedef{PY@tok@sa}{\def\PY@tc##1{\textcolor[rgb]{0.73,0.13,0.13}{##1}}}
\@namedef{PY@tok@sb}{\def\PY@tc##1{\textcolor[rgb]{0.73,0.13,0.13}{##1}}}
\@namedef{PY@tok@sc}{\def\PY@tc##1{\textcolor[rgb]{0.73,0.13,0.13}{##1}}}
\@namedef{PY@tok@dl}{\def\PY@tc##1{\textcolor[rgb]{0.73,0.13,0.13}{##1}}}
\@namedef{PY@tok@s2}{\def\PY@tc##1{\textcolor[rgb]{0.73,0.13,0.13}{##1}}}
\@namedef{PY@tok@sh}{\def\PY@tc##1{\textcolor[rgb]{0.73,0.13,0.13}{##1}}}
\@namedef{PY@tok@s1}{\def\PY@tc##1{\textcolor[rgb]{0.73,0.13,0.13}{##1}}}
\@namedef{PY@tok@mb}{\def\PY@tc##1{\textcolor[rgb]{0.40,0.40,0.40}{##1}}}
\@namedef{PY@tok@mf}{\def\PY@tc##1{\textcolor[rgb]{0.40,0.40,0.40}{##1}}}
\@namedef{PY@tok@mh}{\def\PY@tc##1{\textcolor[rgb]{0.40,0.40,0.40}{##1}}}
\@namedef{PY@tok@mi}{\def\PY@tc##1{\textcolor[rgb]{0.40,0.40,0.40}{##1}}}
\@namedef{PY@tok@il}{\def\PY@tc##1{\textcolor[rgb]{0.40,0.40,0.40}{##1}}}
\@namedef{PY@tok@mo}{\def\PY@tc##1{\textcolor[rgb]{0.40,0.40,0.40}{##1}}}
\@namedef{PY@tok@ch}{\let\PY@it=\textit\def\PY@tc##1{\textcolor[rgb]{0.24,0.48,0.48}{##1}}}
\@namedef{PY@tok@cm}{\let\PY@it=\textit\def\PY@tc##1{\textcolor[rgb]{0.24,0.48,0.48}{##1}}}
\@namedef{PY@tok@cpf}{\let\PY@it=\textit\def\PY@tc##1{\textcolor[rgb]{0.24,0.48,0.48}{##1}}}
\@namedef{PY@tok@c1}{\let\PY@it=\textit\def\PY@tc##1{\textcolor[rgb]{0.24,0.48,0.48}{##1}}}
\@namedef{PY@tok@cs}{\let\PY@it=\textit\def\PY@tc##1{\textcolor[rgb]{0.24,0.48,0.48}{##1}}}

\def\PYZbs{\char`\\}
\def\PYZus{\char`\_}
\def\PYZob{\char`\{}
\def\PYZcb{\char`\}}

\def\PYZhy{\char`\-}
\def\PYZsq{\char`\'}
\def\PYZdq{\char`\"}

\makeatother

    \makeatletter
        \newbox\Wrappedcontinuationbox
        \newbox\Wrappedvisiblespacebox
        \newcommand*\Wrappedvisiblespace {\textcolor{red}{\textvisiblespace}}
        \newcommand*\Wrappedcontinuationsymbol {\textcolor{red}{\llap{\tiny$\m@th\hookrightarrow$}}}
        \newcommand*\Wrappedcontinuationindent {3ex }
        \newcommand*\Wrappedafterbreak {\kern\Wrappedcontinuationindent\copy\Wrappedcontinuationbox}
        \newcommand*\Wrappedbreaksatspecials {%
            \def\PYGZus{\discretionary{\char`\_}{\Wrappedafterbreak}{\char`\_}}%
            \def\PYGZob{\discretionary{}{\Wrappedafterbreak\char`\{}{\char`\{}}%
            \def\PYGZcb{\discretionary{\char`\}}{\Wrappedafterbreak}{\char`\}}}%
            \def\PYGZca{\discretionary{\char`\^}{\Wrappedafterbreak}{\char`\^}}%
            \def\PYGZam{\discretionary{\char`\&}{\Wrappedafterbreak}{\char`\&}}%
            \def\PYGZlt{\discretionary{}{\Wrappedafterbreak\char`\<}{\char`\<}}%
            \def\PYGZgt{\discretionary{\char`\>}{\Wrappedafterbreak}{\char`\>}}%
            \def\PYGZsh{\discretionary{}{\Wrappedafterbreak\char`\#}{\char`\#}}%
            \def\PYGZpc{\discretionary{}{\Wrappedafterbreak\char`\%}{\char`\%}}%
            \def\PYGZdl{\discretionary{}{\Wrappedafterbreak\char`\$}{\char`\$}}%
            \def\PYGZhy{\discretionary{\char`\-}{\Wrappedafterbreak}{\char`\-}}%
            \def\PYGZsq{\discretionary{}{\Wrappedafterbreak\textquotesingle}{\textquotesingle}}%
            \def\PYGZdq{\discretionary{}{\Wrappedafterbreak\char`\"}{\char`\"}}%
            \def\PYGZti{\discretionary{\char`\~}{\Wrappedafterbreak}{\char`\~}}%
        }
        \newcommand*\Wrappedbreaksatpunct {%
            \lccode`\~`\.\lowercase{\def~}{\discretionary{\hbox{\char`\.}}{\Wrappedafterbreak}{\hbox{\char`\.}}}%
            \lccode`\~`\,\lowercase{\def~}{\discretionary{\hbox{\char`\,}}{\Wrappedafterbreak}{\hbox{\char`\,}}}%
            \lccode`\~`\;\lowercase{\def~}{\discretionary{\hbox{\char`\;}}{\Wrappedafterbreak}{\hbox{\char`\;}}}%
            \lccode`\~`\:\lowercase{\def~}{\discretionary{\hbox{\char`\:}}{\Wrappedafterbreak}{\hbox{\char`\:}}}%
            \lccode`\~`\?\lowercase{\def~}{\discretionary{\hbox{\char`\?}}{\Wrappedafterbreak}{\hbox{\char`\?}}}%
            \lccode`\~`\!\lowercase{\def~}{\discretionary{\hbox{\char`\!}}{\Wrappedafterbreak}{\hbox{\char`\!}}}%
            \lccode`\~`\/\lowercase{\def~}{\discretionary{\hbox{\char`\/}}{\Wrappedafterbreak}{\hbox{\char`\/}}}%
            \catcode`\.\active
            \catcode`\,\active
            \catcode`\;\active
            \catcode`\:\active
            \catcode`\?\active
            \catcode`\!\active
            \catcode`\/\active
            \lccode`\~`\~
        }
    \makeatother

    \let\OriginalVerbatim=\Verbatim
    \makeatletter
    \renewcommand{\Verbatim}[1][1]{%
        \sbox\Wrappedcontinuationbox {\Wrappedcontinuationsymbol}%
        \sbox\Wrappedvisiblespacebox {\FV@SetupFont\Wrappedvisiblespace}%
        \def\FancyVerbFormatLine ##1{\hsize\linewidth
            \vtop{\raggedright\hyphenpenalty\z@\exhyphenpenalty\z@
                \doublehyphendemerits\z@\finalhyphendemerits\z@
                \strut ##1\strut}%
        }%
        \def\FV@Space {%
            \nobreak\hskip\z@ plus\fontdimen3\font minus\fontdimen4\font
            \discretionary{\copy\Wrappedvisiblespacebox}{\Wrappedafterbreak}
            {\kern\fontdimen2\font}%
        }%

        \Wrappedbreaksatspecials
        \OriginalVerbatim[#1,codes*=\Wrappedbreaksatpunct]%
    }
    \makeatother

    \definecolor{incolor}{HTML}{303F9F}
    \definecolor{outcolor}{HTML}{D84315}
    \definecolor{cellborder}{HTML}{CFCFCF}
    \definecolor{cellbackground}{HTML}{F7F7F7}

    \makeatletter
    \newcommand{\boxspacing}{\kern\kvtcb@left@rule\kern\kvtcb@boxsep}
    \makeatother
    \newcommand{\prompt}[4]{
        {\ttfamily\llap{{\color{#2}[#3]:\hspace{3pt}#4}}\vspace{-\baselineskip}}
    }

\title{The Approximation Ratio for the Risk of Myopic Bayesian Active Learning for Linear Regression}

\author{%
  Stephen Mussmann\\
  School of Computer Science\\
  Georgia Institute of Technology\\
  \texttt{mussmann@gatech.edu} \\
}

\begin{document}

\maketitle

\begin{abstract}
  Active learning studies the fundamental question: what data should we choose to observe? The greedy algorithm in optimal experiment design is a common heuristic and also equivalent to \emph{myopic} Bayesian active learning for linear regression, the common framework where long-term planning is replaced with the one-step optimal choice. In this work, we prove a first-of-its-kind approximation ratio for the greedy algorithm's risk that is tight up to an absolute constant. The approximation ratio is linear in the \emph{maximum initial leverage score} (MILS), a newly identified quantity fundamental to the greedy algorithm's performance. Finally, we illustrate the results with simple numerical simulations.
\end{abstract}

\section{Introduction}
 
Optimal experimental design and active learning are two closely-related paradigms for selecting informative data under a budget constraint. Both ask the same fundamental question: given the freedom to choose from a pool of candidate inputs, which should we observe in order to minimize model parameter estimation error or future prediction error? In experimental design the choices are typically made offline, while in active learning they are made adaptively as new labels are observed. However, for Bayesian linear regression, the model  under consideration in this work, the observation values don't affect the estimation or prediction risk, and thus the offline and adaptive settings are equivalent. The task is therefore a particular set optimization function ($A/V$-optimal design) with a large search space of size $\binom{n}{k}$ that precludes brute force in any realistic setting.

While exact optimization is NP-hard \citep{li2025strong}, approximation algorithms have been developed. In this work, we focus on the greedy algorithm. The greedy algorithm starts from the empty set and repeatedly adds the point that yields the largest immediate reduction in risk, until $k$ points have been selected. Not only is the greedy algorithm a practical choice in the offline setting, but more importantly, addresses a fundamental question in the adaptive setting. In the adaptive setting, to avoid the intractability of planning, the most common Bayesian active learning algorithms~\citep{mackay1992information,gal2017deep,smith2023prediction} rely on a myopic approach: choose the observation that optimally reduces the risk, as if it were the last step. The connection between ``one-step optimal'' and ``multi-step optimal'' is a gap with little work in the literature. We focus on that connection via analyzing the equivalent greedy algorithm.

Surprisingly little is known about how the greedy algorithm compares to the optimal strategy for this problem. Existing guarantees \citep{pmlr-v70-bian17a,chamon2017approximate} bound the the \emph{reduction} in the estimation or prediction risk ($A/V$-optimal design). By showing the risk reduction is monotone and approximately submodular, these work show greedy attains a constant fraction of the optimal \emph{reduction}. However, a constant factor approximation factor for the \emph{reduction} is often vacuous. Whether greedy achieves a constant factor approximation for the risk itself has, to our knowledge, remained open.
 
We close this gap with the following contributions:
\begin{itemize}
    \item \textbf{Constant-factor risk guarantee.} We prove that the reciprocal risk for Bayesian Linear Regression is approximately submodular in the sense of~\citet{das2018approximate}. Combining this with existing analyses of greedy under approximate submodularity yields the first constant-factor approximation guarantee on the risk achieved by greedy. The constant is the problem-dependent quantity \emph{maximum initial leverage score} (MILS).
    \item \textbf{Parametrized hard instance showing tightness.} We construct a family of problems on which greedy's risk is provably a factor of the MILS larger than the optimal risk which matches our upper bound. This shows that the problem-dependent factor in our guarantee is necessary, not an artifact of the analysis.
    \item \textbf{Numerical simulation.} We use numerical simulation to illustrate previously known bounds and our bound, as well as confirm greedy's poor performance in our construction.
\end{itemize}

We provide the precise problem statement in Section~\ref{sec:problem}, cover necessary background and related work in Section~\ref{sec:background}, provide our upper bound and lower bounds in Sections \ref{sec:upper} and \ref{sec:lower}, then show an illustrative example in Section~\ref{sec:example}.

\section{Problem Statement}
\label{sec:problem}
Precisely, our problem statement is the following:
\begin{problem}
\label{problem}
    Given a set of $n$ vectors  $V = \{v_i\}_{i=1}^n \subset \mathbb{R}^d$, a positive definite matrix $\Lambda \in \mathbb{R}^{d \times d}$, and a budget $k$, choose a set $S \subset [n]$ of size $|S|=k$ to minimize 

    \begin{align}
        f(S) := \tr\left( \left(\Lambda + \sum_{i \in S} v_i v_i^\top \right)^{-1} \right)
    \end{align}
\end{problem}

The problem parameters are integers $n$, $k$, and $d$. In our analysis, we show the importance of an additional problem-dependent parameter which we refer to as the \emph{maximum initial leverage score} (MILS),
\begin{align}
    \mils = \max_{i \in [n]} v_i^\top \Lambda^{-1} v_i.
\end{align}

For a given problem defined by $V$, $\Lambda$, and $k$, we denote an optimal solution as
\begin{align}
    S^\star = \arg\max_{S \subset [n]: |S|=k} f(S)
\end{align}

\subsection{Greedy Algorithm}

In this paper, we analyze the greedy algorithm, which, given a problem, returns a set $S_\text{greedy}$. The greedy algorithm starts with an initial $S_0=\emptyset$, then for $k$ iterations, chooses the element that when added, would minimize $f$. See Algorithm~\ref{alg}. 

\begin{algorithm}[H]
\caption{Greedy Algorithm}
\label{alg}
\renewcommand{\algorithmicrequire}{\textbf{Input:}}
\renewcommand{\algorithmicensure}{\textbf{Output:}}
\begin{algorithmic}[1]
\REQUIRE Vectors $V \in \mathbb{R}^{n \times d}$, Matrix $\Lambda \in \mathbb{R}^{d \times d}$, budget $k \in \mathbb{N}$
\ENSURE Selected set $S \subset [n]$ of size $k$
\STATE $S_0 = \emptyset$
\FOR{iteration $t \leftarrow 1$ \TO $k$}
\STATE Compute $i_t \in \arg\min_{i \not\in S_{t-1}} f(S_{t-1} \cup \{i\})$
\STATE Set $S_t = S_{t-1} \cup \{i_t\}$
\ENDFOR
\STATE \textbf{return} $S_\text{greedy} = S_k$
\end{algorithmic}
\end{algorithm}

Note that there is non-determinacy in the case of ties. When we prove a result for the greedy algorithm, we require that the theorem holds for any choice of tie-breaking.

For the offline problem, this algorithm is attractive due to its simplicity and computational efficiency. The greedy algorithm runs in time $\mathcal{O}(nk)$. In active learning with linear regression, the greedy algorithm is equivalent to the myopic algorithm which is attractive due to removing the need for planning.
 
\section{Background and Related Work}
\label{sec:background}
\subsection{Active Learning}

Active learning studies the setting where there is a large pool of unlabeled data and a limited labeling budget. An active learning algorithm adaptively chooses which points to label next in order to achieve the best test performance. Several recent approaches study the myopic Bayesian setting \citep{gal2017deep,kirsch2019batchbald,mussmann2022active,smith2023prediction}, where the next point or batch of points is chosen by minimizing an expected cost (e.g., loss, entropy). Given the computational planning of planning, these methods only minimize the cost after a single step, similar to greedy algorithms.

\subsection{Bayesian Linear Regression}

Our primary motivation for the setting is Bayesian linear regression \citep{bishop2006pattern,pml2Book}. This model is defined by a prior covariance $\Sigma_0$, observation noise $\sigma^2$, and a fixed set of points $X = \{x_i\}_{i=1}^n \subset \mathbb{R}^d$. Then, the model is
\begin{align}
    \theta &\sim \mathcal{N}(0,\Sigma_0) \\
    \epsilon_i &\sim \mathcal{N}(0,\sigma^2) \\
    y_i &= x_i^\top \theta + \epsilon_i
\end{align}

A standard result is that the parameter posterior is $\theta | \{(x_i,y_i)\}_{i \in S} \sim \mathcal{N}(\mu_S, \Sigma_S)$ where
\begin{align}
    \Sigma_S &= \left(\Sigma_0^{-1} + \sigma^{-2} \textstyle\sum_{i \in S} x_i x_i^\top\right)^{-1} \\
    \mu_S &=\sigma^{-2} \Sigma_S \sum_{i \in S} y_i x_i
\end{align}

Two natural criteria for chosing $S$ are to minimize the variance of the posterior $\theta$ or to minimize the average variance of predictions on test points $\{\overline{x}_i\}_{i=1}^m$ \citep{chaloner1995bayesian}. In the first case,
\begin{align}
    \mathbb{E}_{\theta|S} \left[\|\theta - \mathbb{E}_{\theta|S}[\theta]\|^2\right] = \tr(\Sigma_S)
\end{align}

In the second case,
\begin{align}
    \frac{1}{m} \sum_{i=1}^m \text{Var}_{\theta|S}\left[\theta\cdot \overline{x}_i\right] = \tr\left(\Sigma_S  \frac{1}{m} \sum_{i=1}^m \overline{x}_i \overline{x}_i^\top\right)
\end{align}

In both cases, we can write the optimization criteria in the form of Problem Statement~\ref{problem}. For the first criteria, with $\Lambda = \Sigma_0$ and $v_i = \sigma^{-1} x_i$, $f(S) = \tr(\sigma_S)$.
For the second criteria, if $\Sigma_{\overline{X}} = \frac{1}{m}  \sum_{i=1}^m \overline{x}_i \overline{x}_i^\top$ is full-rank, then with $\Lambda = \Sigma_{\overline{X}}^{-1/2} \Sigma_0 \Sigma_{\overline{X}}^{-1/2}$ and $v_i = \sigma^{-1} \Sigma_{\overline{X}}^{-1/2} x_i$, $f(S) = \tr(\Sigma_S \Sigma_{\overline{X}})$.

Notably, in this case, since the objective value doesn't depend on the observation value $y_i$, only that it is observed, adaptivity serves no role. Therefore, active learning in this setting is equivalent to set optimization, removing a source of complexity for algorithmic analysis.

\subsection{Optimal Experimental Design}
Optimal experimental design is the classical statistical problem of choosing inputs at which to observe a response so as to most accurately estimate a parameter or prediction~\citep{kiefer1959optimum,fedorov2013theory,atkinson2007optimum}. Our goal is to select a set $S$ of design points $x_i$ in a candidate pool $X = \{x_1, \dots, x_n\}$. The two design criteria we consider are \emph{A-optimality}, a frequentist version of minimizing the parameter variance, and \emph{V-optimality}, a frequentist version of minimizing the prediction variance on a set of points. With $\lambda$-strength L2 regularization, the A-optimality criteria is $\mathrm{tr}((\lambda I + \sum_{i \in S} x_i x_i^\top)^{-1})$ and the V-optimality criteria for points $\{\overline{x}_i\}_{i=1}^m$ with $L = \frac{1}{m} \sum_{i=1}^m \overline{x}_i \overline{x}_i^\top$  is $\mathrm{tr}((\lambda I + \sum_{i \in S} x_i x_i^\top)^{-1}L)$. The positive definite matrix $L$ encodes the importance of estimation in different parameter directions.

In both cases, we can write them in our Problem Statement~\ref{problem}. For $A$-optimaliy, $\Lambda = \lambda I_d$ and $v_i = x_i$, and for $V$-optimality $\Lambda = \lambda L^{-1}$ and $v_i = L^{-1/2} x_i$. Solving either exactly under a cardinality constraint is NP-hard~\citep{li2025strong}, motivating both convex-relaxation~\citep{allen2021near,nikolov2022proportional} and combinatorial approaches \citep{madan2019combinatorial}.
 
\subsection{Submodularity, Approximate Submodularity, and Curvature}
A set function $g: 2^{[n]} \to \mathbb{R}$ is \emph{submodular} if for all $S \subseteq T \subseteq [n]$ and $i \notin T$,
\begin{equation*}
g(S \cup \{i\}) - g(S) \;\geq\; g(T \cup \{i\}) - g(T),
\end{equation*}
the diminishing-returns property. The classical theorem of \citet{nemhauser1978analysis} states that for monotone submodular $g$ with $g(\emptyset) = 0$, greedy achieves a $(1 - 1/e)$-approximation to the cardinality-constrained maximum $S^\star \in \arg\max_{S \subset [n], |S|=k} g(S)$, 
\begin{align}
    \frac{g(S_\text{greedy})}{g(S^\star)} \geq 1 - 1/e
\end{align}

For set functions that are not submodular, two parameters quantify whether the set function is approximately submodular.

\citet{das2018approximate} introduces the \emph{submodularity ratio}, 
\begin{equation}
\gamma(g) = \min_{S, L \subset [n]: S \cap L = \emptyset}\frac{\sum_{i \in S} \big(g(L \cup \{i\}) - g(L)\big)}{g(L \cup S) - g(L)},
\end{equation}

A set function is submodular exactly when $\gamma = 1$. \citet{pmlr-v70-bian17a} defines a \emph{curvature} that bounds the extent to which the marginal gain of an element shrinks as more elements are added

\begin{align}
    \alpha = 1 - \min_{S \subset T, i \not\in T} \frac{g(T \cup \{i\}) - g(T)}{g(S \cup \{i\}) - g(S)}
\end{align}

Note that $\gamma, \alpha \in [0,1]$. \citet{pmlr-v70-bian17a} shows that for monotone non-negative $g$ with submodularity ratio $\gamma$ and curvature $\alpha$, greedy attains a $\frac{1}{\alpha}\bigl(1 - e^{-\alpha\gamma}\bigr)$-approximation, recovering the $(1 - 1/e)$ bound when $\gamma = \alpha = 1$.
 
\subsection{Submodularity and A/V-Optimal Design}
To the best of our knowledge, the only existing approximation guarantees for the greedy algorithm applied to A-optimal design are found in  \citet{pmlr-v70-bian17a} and \citet{chamon2017approximate}. In both cases, the object of analysis is the reduction
\begin{equation}
\label{eq:reduction}
F_{\text{reduction}}(S) \;=\; f(\emptyset) - f(S),
\end{equation}
which is monotone non-decreasing, satisfies $F_{\text{reduction}}(\emptyset) = 0$. \citet{pmlr-v70-bian17a} provides a bound on the submodularity ratio and curvature, which yields a multiplicative approximation factor for $F_\text{reduction}$. \citet{chamon2017approximate} analysis $-F_\text{reduction}$ by their defined $\alpha$-supermodularity for monotone non-increasing functions, which can be equivalently defined as $\alpha$-submodularity. Intuitively, \citet{chamon2017approximate} defines the $\alpha$-submodularity as the minimal ratio of marginal gains as a function of the cardinalities of $S$ and $T$, and uses it to prove a multiplicative approximation ratio for the greedy algorithm. They then provide specific values of $\alpha$ for $-F_\text{reduction}$.
 
This style of guarantee has a basic limitation: it bounds how much risk has been removed, not how much remains. If $f(\emptyset)$ is sufficiently large, $F_\text{reduction}(S_\text{greedy})$ could be within a constant factor of $F_\text{reduction}(S^\star)$, while the objective itself $f(S_\text{greedy})$ could be an arbitrarily large factor away from $f(S^\star)$. For example, if $f(\emptyset)=1$, $f(S^\star)=\epsilon$, and $f(S_\text{greedy}) = \frac{1}{2}$, then $F_\text{reduction}(S_\text{greedy})/F_\text{reduction}(S^\star) \geq 1/2$ but $f(S^\star)/f(S_\text{greedy}) = 2\epsilon$. Consequently, prior to the present work, no approximation ratio guarantee on the achieved risk $f$ itself was known for greedy A- or V-optimal design.

\section{Upper Bound}
\label{sec:upper}
Previous work analyzes the risk reduction function $F_\text{reduction}(S) = f(\emptyset) - f(S)$. Instead, we focus on the reciprocal risk function, $F_\text{reciprocal}(S) = 1/f(S)$. Both these functions are monotonic decreasing functions of $f$, so the greedy algorithm remains unchanged. Our main result hinges on a lower bound on the submodularity ratio,

\begin{lemma}
\label{thm:submod}
    The submodularity ratio of $F_\text{reciprocal}$ is bounded as $\gamma(F_\text{reciprocal}) \geq (1+\mils)^{-1}$
\end{lemma}

Combined with the result from \citet{das2018approximate} (or from \citet{pmlr-v70-bian17a} with $\alpha=1$), we find that,

\begin{corollary}
    \begin{align}
        \frac{F_\text{reciprocal}(S_\text{greedy})}{F_\text{reciprocal}(S^\star)} \geq 1 - \exp(-(1+\mils)^{-1})
    \end{align}
\end{corollary}

The main result then follows from taking the reciprocal of the equation, which yields a complicated expression. The following algebraic proposition creates a more interpretable bound,
\begin{proposition}
\label{prop:algebra}
For any $h\geq 0$,
\begin{align}
    \frac{1}{1 - \exp(-(1+h)^{-1})} \leq h + \frac{1}{1-1/e} \leq h + 1.582
\end{align}
\end{proposition}

The proof is in Appendix~\ref{app:proof-prop-algebra}. Our main result is then,

\begin{theorem}
    \begin{align}
        \frac{f(S_\text{greedy})}{f(S^\star)} \leq \mils + \frac{1}{1-1/e}
    \end{align}
\end{theorem}

Thus, greedy acheives a risk approximation guarantee of $h_\text{max}+1.582$, which scales linearly in the Maximum Initial Leverage Score (MILS).

\subsection{Tightness of Approximate Submodularity}

In this section, we show that our bound on $\gamma$ is tight and that the curvature is arbitrarily close to $1$. Note that approximate submodularity and curvature don't depend on the budget $k$ so we drop these from the following statements.

\begin{proposition}
\label{prop:gamma}
For any $n \in \mathbb{N}$ and $h > 0$, there exists a problem with $n$ vectors, $h_\text{max} = h$ and $d=n$ dimensions, such that the submodularity ratio is  $\gamma(F_\text{reciprocal}) \leq \frac{1}{(1+h)-\frac{h}{n}}$
\end{proposition}

The proof is in Appendix~\ref{app:proof-prop-gamma}. Note that as $n \rightarrow \infty$, $\gamma \rightarrow (1+h_\text{max})^{-1}$ implying that Theorem~\ref{thm:submod} is tight (when examining all values of $n$). Also, note that for all $h>0$ and $d \geq 2$, the submodularity ratio is strictly less than $1$, so $F_\text{reciprocal}$ is not submodular \citep{das2018approximate}. Furthermore, the curvature $\alpha$ for $F_\text{reciprocal}$ is not in general bounded away from $1$, so the more refined analysis using curvature \citep{pmlr-v70-bian17a} doesn't provide an improvement.

\begin{proposition}
\label{prop:alpha}
For any $n \in \mathbb{N}$ and $h > 0$, there exists a problem with $n$ vectors, $h_\text{max} = h$, and $d=2$ where the curvature of $F_\text{reciprocal}$ is greater than $1 - \left(\frac{1}{1+\frac{(n-1)h}{2}}\right)^2$
\end{proposition}

The proof is in Appendix~\ref{app:proof-prop-alpha}. Note the curvature approaches $1$ as $n \rightarrow \infty$.

\subsection{Proof of Lemma \ref{thm:submod}}

In this section, we prove Lemma~\ref{thm:submod} to by showing $F_\text{reciprocal}$ is approximately submodular. First, we prove a specific matrix identity.

\begin{lemma}
\label{lem:matrix}
    For any symmetric positive definite matrix $A$ and symmetric positive semi-definite matrix $B$
    \begin{align}
        \tr\left(A^{-1}\right) - \tr\left(\left(A+B)\right)^{-1}\right) \leq \frac{\tr\left(A^{-1}BA^{-1}\right) \tr\left(\left(A+B\right)^{-1}\right)}{\tr\left(A^{-1}\right)}
    \end{align}
\end{lemma}
\begin{proof}
    From the matrix version of the Cauchy–Schwarz inequality,
    \begin{align}
        \tr\left(X^T Y\right)^2 \leq \tr\left(X^T X\right) \tr\left(Y^T Y\right)
    \end{align}

    Let $X = (A+B)^{-1/2}$ and $Y=(A+B)^{1/2} A^{-1}$.
    \begin{align}
        \tr\left(A^{-1}\right)^2 &\leq \tr\left((A+B\right)^{-1}) \tr\left(A^{-1} (A + B) A^{-1}\right) \\
        &= \tr\left((A+B)^{-1}\right) \left[\tr\left(A^{-1}\right) + \tr\left(A^{-1} B A^{-1}\right) \right]
    \end{align}

    Dividing both sides by $\tr\left(A^{-1}\right)$ and subtracting $\tr\left((A+B)^{-1}\right)$ from both sides yields the result.
\end{proof}

The following lemma is a restated version of Lemma~\ref{thm:submod} in different notation with $\gamma(F_\text{reciprocal})$ written out.

\begin{lemma-repeat}[\ref{thm:submod} (restated)]
For any $L, S \subset [n]$ with $L \cap S = \emptyset$,
    \begin{align}
         \frac{\sum_{i \in S} F_\text{reciprocal}(L \cup \{i\}) - F_\text{reciprocal}(L)}{F_\text{reciprocal}(L \cup S) - F_\text{reciprocal}(L)} \geq (1+h_\text{max})^{-1}  \
    \end{align}
\end{lemma-repeat}
\begin{proof}
    Define $A = \Lambda + \sum_{i \in L} v_i v_i^\top$ and $B = \sum_{i \in S} v_i v_i^\top$.
    \begin{align}
    F_\text{reciprocal}(L \cup S) - F_\text{reciprocal}(L) &= \frac{1}{f(L \cup S)} - \frac{1}{f(L)} \\
    &= \frac{\text{Tr}(A^{-1}) - \text{Tr}((A+B)^{-1})}{\text{Tr}(A^{-1}) \text{Tr}((A+B)^{-1})} \\
    &\leq \frac{\text{Tr}(A^{-1} B A^{-1})}{\text{Tr}(A^{-1})^2} \\
        &= \sum_{i \in S} \frac{\tr\left(A^{-1} v_i v_i^\top A^{-1}\right)}{\tr\left(A^{-1}\right)^2} \\
        &\leq \sum_{i \in S} \frac{(1+\mils)}{1 + v_i^\top A^{-1} v_i} \frac{\tr\left(A^{-1} v_i v_i^\top A^-1\right)}{ \tr\left(A^{-1}\right)^2} \\
        &= (1+\mils) \sum_{i \in S} \frac{\tr\left(A^{-1}\right) - \tr\left((A+v_iv_i^\top\right)^{-1})}{\tr\left(A^{-1}\right)^2} \\
        &\leq (1+\mils) \sum_{i \in S} \frac{\tr\left(A^{-1}\right) - \tr\left((A+v_iv_i^\top\right)^{-1})}{\tr\left(A^{-1}\right) \tr\left((A+v_iv_i^\top)^{-1}\right)} \\
        &= (1+\mils) \sum_{i \in S} \frac{1}{f(L \cup \{i\})} - \frac{1}{f(L)} \\
        &= (1+\mils) \sum_{i \in S} F_\text{reciprocal}(L \cup \{i\}) - F_\text{reciprocal}(L)
    \end{align}
    
    The lines follow from: definition of $F_\text{reciprocal}$, definition of $f$ and finding a common denominator, Lemma~\ref{lem:matrix}, linearity of $B$ and $\tr$, definition of $\mils$ and monotonicity of $X^{-1}$ for symmetric positive definite matrices, the Sherman-Morrison formula, the monotonicity of $X^{-1}$ for symmetric positive definite matrices, simplifying fractions and definition of $f$, and definition of $F_\text{reciprocal}$. Finally, rearranging terms yields the result.
\end{proof}

\section{Lower Bound}
\label{sec:lower}
We now construct an explicit problem showing that the upper bound is tight up to constants. 

\begin{theorem}
\label{thm:lower}
    For any $h>0$ and $d \geq 4$, if there exists an order $d$ Hadamard matrix, then there exists a $d$-dimensional problem with $n=2d$ vectors such that $\mils = \max(h,4)$ and for a cardinality constraint of $k=d$,
    \begin{align}
        \frac{f(S_\text{greedy})}{f(S^\star)} \geq \frac{1+h}{5}
    \end{align}
\end{theorem}

The proof is in Appendix~\ref{app:lower-bound}. Note that by using Sylvester's construction, there are Hadamard matrices for any order that is a power of $2$. To make it explicit, the example for $d=4$ is

\begin{align*}
    \Lambda &= \begin{bmatrix}
        \exp\left(\frac{0}{4}\right) & 0 & 0 & 0 \\
        0 & \exp\left(\frac{1}{4}\right) & 0 & 0 \\
        0 & 0 & \exp\left(\frac{2}{4}\right) & 0 \\
        0 & 0 & 0 & \exp\left(\frac{3}{4}\right)
    \end{bmatrix} \\
    v_1 &= \begin{bmatrix}
        2 \exp\left(\frac{0}{8}\right) & 0 & 0 & 0
    \end{bmatrix} \\
    v_2 &= \begin{bmatrix}
        0 & 2 \exp\left(\frac{1}{8}\right) & 0 & 0
    \end{bmatrix} \\
    v_3 &= \begin{bmatrix}
        0 & 0 & 2 \exp\left(\frac{2}{8}\right) & 0
    \end{bmatrix} \\
    v_4 &= \begin{bmatrix}
        0 & 0 & 0 & 2 \exp\left(\frac{3}{8}\right)
    \end{bmatrix} \\
    v_5 &= \begin{bmatrix}
        \frac{\sqrt{h} \exp\left(\frac{0}{8}\right)}{2} & \frac{\sqrt{h} \exp\left(\frac{1}{8}\right)}{2} & \frac{\sqrt{h} \exp\left(\frac{2}{8}\right)}{2} & \frac{\sqrt{h} \exp\left(\frac{3}{8}\right)}{2}
    \end{bmatrix} \\
    v_6 &= \begin{bmatrix}
        \frac{\sqrt{h} \exp\left(\frac{0}{8}\right)}{2} & -\frac{\sqrt{h} \exp\left(\frac{1}{8}\right)}{2} & \frac{\sqrt{h} \exp\left(\frac{2}{8}\right)}{2} & -\frac{\sqrt{h} \exp\left(\frac{3}{8}\right)}{2}
    \end{bmatrix} \\
    v_7 &= \begin{bmatrix}
        \frac{\sqrt{h} \exp\left(\frac{0}{8}\right)}{2} & \frac{\sqrt{h} \exp\left(\frac{1}{8}\right)}{2} & -\frac{\sqrt{h} \exp\left(\frac{2}{8}\right)}{2} & -\frac{\sqrt{h} \exp\left(\frac{3}{8}\right)}{2}
    \end{bmatrix} \\
    v_8 &= \begin{bmatrix}
        \frac{\sqrt{h} \exp\left(\frac{0}{8}\right)}{2} & -\frac{\sqrt{h} \exp\left(\frac{1}{8}\right)}{2} & -\frac{\sqrt{h} \exp\left(\frac{2}{8}\right)}{2} & \frac{\sqrt{h} \exp\left(\frac{3}{8}\right)}{2}
    \end{bmatrix}  
\end{align*}

where the greedy algorithm will choose $\{1,2,3,4\}$ which is outperformed by choosing $\{5,6,7,8\}$.

We ran numerical experiments for $h=10,100,1000$ confirming that $h_\text{max} = 10,100,1000$ and $\frac{f(S_g)}{f(S^*)} = \frac{11}{5}, \frac{101}{5}, \frac{1001}{5}$. The code snippet and output is shown in Appendix~\ref{app:lower-bound-experiment}.

While Theorem~\ref{thm:lower} is stated with simple constants, in Appendix~\ref{app:lower-bound}, we prove a slightly more general result which yields slightly better constants for large $d$. From numerical results, it appears our construction cannot yield examples with approximation ratios lower than $\frac{1+h}{3.23}$.

\section{Illustrative Numerical Example}
\label{sec:example}
For illustration of our bounds, we consider a simple setting with $\Lambda = I_d$ and $x_i$ drawn uniformly from the unit sphere $S^{d-1}$. We then run greedy with $d=20$, $n=1000$, and $k\in[100]$. For each iteration, we can compute a bound on $f(S^\star)$. For an approximation ratio lower bound of $c$ on $F_\text{reciprocal}$ (e.g., Theorem~\ref{thm:submod}),
\begin{align}
    f(S^\star) \geq c f(S_\text{greedy})
\end{align}

For an approximation ratio lower bound of $c$ on $F_\text{reduction}$, (e.g., as in \citet{pmlr-v70-bian17a} or \citet{chamon2017approximate}),
\begin{align}
    f(S^\star) \geq f(\emptyset) - \frac{f(\emptyset)-f(S_\text{greedy})}{c}
\end{align}

This experiment was implemented (code and results in Appendix~\ref{app:illustrative-example}) and the results are shown in Figure~\ref{fig}. We can see that the bound based on $F_\text{reduction}$ from \citet{chamon2017approximate} becomes vacuous around $k=10$. The numerical results conclude that the bound from \citet{pmlr-v70-bian17a} (not shown in figure) is vacuous even for $k=1$ due to the $\gamma$ and $\alpha$ bounds being very close to $0$ and $1$, respectively.

\begin{figure}[h]
\centerline{\includegraphics[scale=0.9]{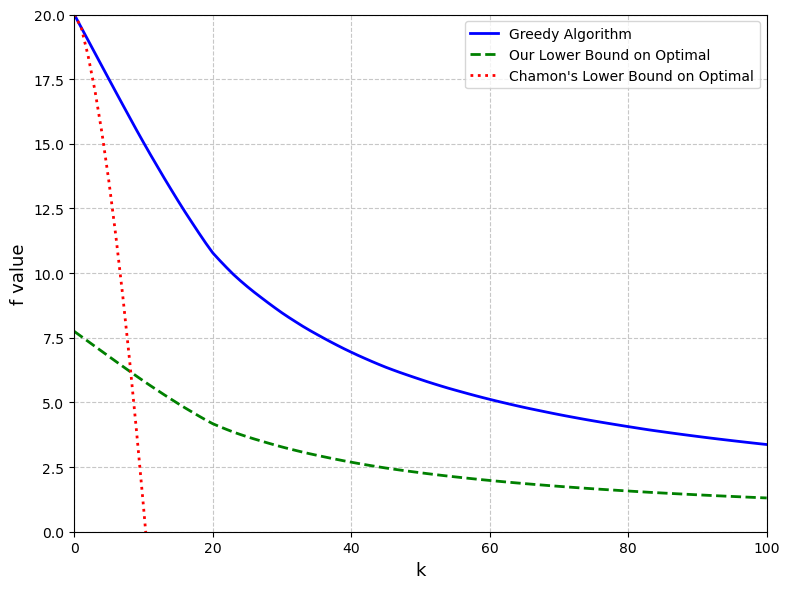}}
\caption{The objective value of the greedy algorithm and the lower bound on the optimal objective from Theorem~\ref{thm:submod} and \citet{chamon2017approximate}.}
\label{fig}
\end{figure}

\section{Discussion}
\label{sec:discussion}

The results in this work not only provide the first $A/V$-optimality criteria approximation ratio guarantee (not on the reduction) for the common greedy heuristic, but addresses a more fundamental question in active learning. Nearly all active learning algorithms avoid planning over multiple data labeling iterations by focusing only on the current data labeling iteration. To our knowledge, there is no existing guarantee on the test loss approximation ratio for myopic algorithms. Here, for Bayesian Logistic Regression, that myopic data labeling nearly matches fully planned data labeling, at least when the MILS is small. We hope this work serves as a foundation for future analyses of other models, especially those where adaptive selection is not equivalent to offline selection.

\bibliographystyle{plainnat}
\bibliography{bibliography}
%%%%%%%%%%%%%%%%%%%%%%%%%%%%%%%%%%%%%%%%%%%%%%%%%%%%%%%%%%%%

\appendix
\newpage

\section{Upper Bound Details}

\subsection{Proof of Proposition~\ref{prop:algebra}}
\label{app:proof-prop-algebra}

\begin{proposition-repeat}[\ref{prop:algebra}]
For any $h\geq 0$,
\begin{align}
    \frac{1}{1 - \exp(-(1+h)^{-1})} \leq h + \frac{1}{1-1/e} \leq h + 1.582
\end{align}
\end{proposition-repeat}
\begin{proof}
    Let $g(h) = \frac{1}{1 - \exp(-(1+h)^{-1})}$ for $h\geq 0$. Then,

    \begin{align}
        g'(h) &= \frac{\exp(-1/(1+h))}{(1-\exp(-1/(1+h)))^2} \frac{1}{(1+h)^2} \\
        &= \frac{\left(\frac{1}{2(1+h)}\right)^2}{\sinh^2 \left(\frac{1}{2(1+h)}\right)} \\
        &\leq 1
    \end{align}

    The last line follows since for $x > 0$, $\sinh(x) \geq x$\footnote{Note that $\sinh(0)=0$ and $\sinh'(x) = \cosh(x) \geq 1$}.
    
    Thus, $g(h) - h$ is non-increasing and so

    \begin{align}
        g(h) \leq h + g(0) = h+\frac{1}{1-1/e}
    \end{align}
\end{proof}

\subsection{Proof of Proposition~\ref{prop:gamma}}
\label{app:proof-prop-gamma}

\begin{proposition-repeat}[\ref{prop:gamma}]
For any $n \in \mathbb{N}$ and $h > 0$, there exists a problem with $n$ vectors, $h_\text{max} = h$ and $d=n$ dimensions, such that the submodularity ratio is  $\gamma(F_\text{reciprocal}) \leq \frac{1}{(1+h)-\frac{h}{n}}$
\end{proposition-repeat}
\begin{proof}
    Let $\Lambda = \frac{1}{h} I_n$ and $v_i = e_i$ for all $i\in [n]$.

    Note that $\frac{1}{\frac{1}{h}+1} = \frac{h}{1+h}$

    \begin{align}
        F_\text{reciprocal}(\emptyset) &= \frac{1}{nh} \\
        F_\text{reciprocal}(\{i\}) &= \frac{1}{(n-1)h + \frac{h}{1+h}} \\
        F_\text{reciprocal}([n]) &= \frac{1+h}{nh} \\        
    \end{align}

Taking the differences,
    \begin{align}
        F_\text{reciprocal}([n]) - F_\text{reciprocal}(\emptyset) &= \frac{1}{n} \\
        \sum_{i=1}^n F_\text{reciprocal}(\{i\}) - F_\text{reciprocal}(\emptyset) &= n \left( \frac{1}{(n-1)h + \frac{h}{1+h}} - \frac{1}{nh} \right) \\
        &= \frac{1}{h - \frac{h^2}{n(1+h)}} - \frac{1}{h} \\
        &= \frac{\frac{h^2}{n(1+h)}}{h^2 - \frac{h^3}{n(1+h)}} \\
        &= \frac{1}{n(1+h) - h}
    \end{align}
    
    Taking the ratio, 
    
    \begin{align}
        \gamma(F_\text{reciprocal}) &\leq \frac{\sum_{i=1}^n F_\text{reciprocal}(\{i\}) - F_\text{reciprocal}(\emptyset)}{F_\text{reciprocal}([n]) - F_\text{reciprocal}(\emptyset)} \\
        &= \frac{n}{n(1+h) - h} \\
        &= \frac{1}{(1+h) - \frac{h}{n}}
    \end{align}
\end{proof}

\subsection{Proof of Proposition~\ref{prop:alpha}}
\label{app:proof-prop-alpha}

\begin{proposition-repeat}[\ref{prop:alpha}]
For any $n \in \mathbb{N}$ and $h > 0$, there exists a problem with $n$ vectors, $h_\text{max} = h$, and $d=2$ where the curvature of $F_\text{reciprocal}$ is greater than $1 - \left(\frac{1}{1+\frac{(n-1)h}{2}}\right)^2$
\end{proposition-repeat}
\begin{proof}
    For convenience, define $m = n-1$. Let $\Lambda =  \frac{1}{h} e_1 e_1^\top + \left(\frac{1}{h} + m\right) e_2 e_2^\top$. Define $v_i = e_1$ for $1 \leq i \leq m$ and $v_{m+1} = e_2$ and . Then,

    \begin{align}
        f_0 :=f(\emptyset) &= \frac{1}{\frac{1}{h} + m} + \frac{1}{\frac{1}{h}} = \frac{h}{1+mh}(2+mh)  \\
        f_1 :=f(\{m+1\}) &= \frac{1}{\frac{1}{h} + m + 1} + \frac{1}{\frac{1}{h}} = \frac{h}{1+mh+h}(2+mh +h) \\
        f_m := f([m]) &= \frac{1}{\frac{1}{h} + m} + \frac{1}{\frac{1}{h} + m} = \frac{h}{1+mh} \cdot 2 \\
        f_{m+1} := f([m+1]) &= \frac{1}{\frac{1}{h} + m + 1} + \frac{1}{\frac{1}{h} + m} = \frac{h}{1+mh+h} \left(2 + \frac{h}{1+mh}\right)\\
    \end{align}

    The curvature is at least

    \begin{align}
        \alpha &\geq 1 - \frac{F_\text{reciprocal}(\{m+1\}) -F_\text{reciprocal}(\emptyset) }{F_\text{reciprocal}([m+1])  - F_\text{reciprocal}([m])} \\
        &= 1 - \frac{1/f_1 - 1/f_0}{1/f_{m+1} - 1/f_m} \\
        &= 1 - \frac{f_0 - f_1}{f_m - f_{m+1}} \frac{f_m f_{m+1}}{f_0 f_1}
    \end{align}

    Because $f_0 - f_1 = \frac{1}{\frac{1}{h} + m} - \frac{1}{\frac{1}{h} + m + 1} = f_n - f_{n+1}$,

    \begin{align}
        \alpha &\geq 1 - \frac{f_n f_{n+1}}{f_0 f_1} \\
        &= 1 - \frac{2 \left(2 + \frac{h}{1+mh}\right)}{(2+mh)(2+mh +h)} \\
        &= 1 - \frac{2 (2 + 2mh + h)}{(1+mh)(2+mh)(2+mh +h)}
    \end{align}

    Note that $\frac{2 + 2mh + x}{2 + mh + x}$ is decreasing for $x\geq 0$. Thus,

    \begin{align}
        \alpha &\geq 1 - \frac{2 (2 + 2mh)}{(1+mh)(2+mh)(2+mh)} \\
        &= 1 - \frac{2^2}{(2+mh)^2} \\
        &= 1 - \left(\frac{1}{1+\frac{(n-1)h}{2}}\right)^2
    \end{align}
\end{proof}

\section{Lower Bound Details}
\label{app:lower-bound}

For the lower bound, we prove a slightly more general result that yields Theorem~\ref{thm:lower} as a corollary.

\begin{lemma}
\label{lem:lower}
    For any $h>0$, $\alpha>0$, $0 < r < 1$, and $d \geq 4$, if there exists an order $d$ Hadamard matrix and
    \begin{align}
        g(d,r,\alpha) = d \ln(1/r) (1+\alpha) - \alpha \ln \left(\frac{(1-r)\alpha^2}{(r^d(1+\alpha)^2 - 1) \ln(1/r)} \right) + \alpha - (\alpha + 2) \frac{\ln(1/r)}{1-r} > 0,
    \end{align}
    then there exists a $d$-dimensional problem with $n=2d$ vectors such that $\mils = \max(h,\alpha)$ and for a cardinality constraint of $k=d$,

    \begin{align}
        \frac{f(S_\text{greedy})}{f(S^\star)} \geq \frac{1+h}{1 + \alpha}
    \end{align}
\end{lemma}
\begin{proof}
    Let $H$ be an order $d$ Hadamard matrix, so all entries are $\pm 1$ and $H H^T = d I_d$.
    
    Let $\Lambda \in \mathbb{R}^{d\times d}$ be a diagonal matrix with diagonal entries $\Lambda_{j,j} = r^{-(j-1)}$.

    We define two sets of vectors, $U = \{u_i\}_{i=1}^d \subset \mathbb{R}^d$ and $W = \{w_i\}_{i=1}^d \subset \mathbb{R}^d$ as follows:
    \begin{align}
        u_{i,j} &= \mathbf{1}[i=j] \sqrt{\alpha r^{-(j-1)}} \\
        w_{i,j} &= H_{j,i} \sqrt{\frac{h r^{-(j-1)}}{d}} \\
    \end{align}
    Let the candidate vectors $V=\{v_i\}_{i=1}^{2d}$ be such that $v_i = u_i$ for $i\in [d]$ and $v_{d+i}=w_i$ for $i \in [d]$.

    Note that,

    \begin{align}
        u_i^\top \Lambda^{-1} u_i &= \sum_{j=1}^d v_{i,j}^2 \Lambda_{j,j}^{-1} = \sum_{j=1}^d \mathbf{1}[i=j] \frac{\alpha r^{-(j-1)}}{r^{-(j-1)}} = \alpha \\
        w_i^\top \Lambda^{-1} w_i &= \sum_{j=1}^d (w_{i,j})^2 \Lambda_{j,j}^{-1} = \sum_{j=1}^d \frac{h r^{-(j-1)}}{d r^{-(j-1)}} = h \\
    \end{align}

    Thus, $h_\text{max} = \max(h,\alpha)$.

    We will later show that the greedy algorithm selects the first $n$ vectors.

    \begin{align}
        f([d]) &= \text{Tr}( (\Lambda + \sum_{i=1}^d u_i u_i^\top)^{-1}) \\
        &= \sum_{j=1}^d (\Lambda_{j,j} + \alpha \Lambda_{j,j})^{-1} \\
    &= (1+\alpha)^{-1} \sum_{j=1}^d r^{j-1}
    \end{align}

    The optimal set of $k=d$ vectors will have $f$ value no larger than $f(\{i\}_{i=d+1}^{2d})$. First, note that

    \begin{align}
         \left[\sum_{i=1}^d w_i w_i^\top\right]_{j,k} &= \sum_{i=1}^d w_{i,j} w_{i,k} \\
         &= \frac{h}{d} \sqrt{r^{-(j-1)}} \sqrt{r^{-(k-1)}} \sum_{i=1}^d  H_{j,i} H_{k,i}  \\
        &= \frac{h}{d} \sqrt{r^{-(j-1)}} \sqrt{r^{-(k-1)}} \left[HH^\top\right]_{j,k} \\
        &= \frac{h}{d} \sqrt{r^{-(j-1)}} \sqrt{r^{-(k-1)}} \left[d I_d\right]_{j,k} \\
        &= \mathbf{1}[j=k] h r^{-(j-1)}
    \end{align}

    Thus, $\sum_{i=1}^d w_i w_i^\top = h \Lambda$
    
    \begin{align}
        f(\{i\}_{i=d+1}^{2d}) &= \text{Tr}( (\Lambda + \sum_{i=1}^d w_i w_i^\top)^{-1}) \\
        &= \text{Tr}( (\Lambda + h \Lambda)^{-1}) \\
        &= (1+h)^{-1} \sum_{j=1}^d r^{j-1}
    \end{align}

    Thus, 

    \begin{align}
        \frac{f([d])}{f(\{i\}_{i=d+1}^{2d})} = \frac{1+h}{1+\alpha}
    \end{align}

    If we can show that the greedy algorithm chooses $[d]$ (the first $k=d$ vectors), then it suffices to prove the result.

    Note that minimizing $f(S_{t-1} \cup \{i_t\})$ is the same as maximizing $f(S_{t-1}) - f(S_{t-1} \cup \{i_t\})$

    Note that for any $k$, if $k < j \leq d$, then from the Sherman-Morrison formula,

    \begin{align}
        f([k]) - f([k] \cup \{j\}) &= \frac{u_j^\top \left(\Lambda + \sum_{i=1}^k u_i u_i^\top\right)^{-2} u_j}{1 + u_j^\top \left(\Lambda + \sum_{i=1}^k u_i u_i^\top\right)^{-1} u_j} \\
        &= \frac{ \sqrt{\alpha r^{-(j-1)}} (r^{-(j-1)})^{-2} \sqrt{\alpha r^{-(j-1)}}}{1 + \sqrt{\alpha r^{-(j-1)}} (r^{-(j-1)})^{-1} \sqrt{\alpha r^{-(j-1)}}} \\
        &= \frac{\alpha r^{j-1}}{(1+\alpha)}
    \end{align}

    which are strictly decreasing in $j$ since $r<1$. 
    
    Thus, as long as $f([k]) - f([k+1]) > f([k]) - f([k] \cup \{d+i\})$ for all $0 \leq k<d$ and $i \in [d]$, then the greedy algorithm will iteratively select sets $[1], [2], [3], \dots, [d]$ and the theorem is proven.

    Note that for $0 \leq k<d$ and $i \in [d]$,

    \begin{align}
        f([k]) - f([k] \cup \{d+i\}) &= \frac{w_i^\top \left[ \Lambda + \sum_{j=1}^k  u_j u_j^\top \right]^{-2} w_i}{1 + w_i^\top \left[ \Lambda + \sum_{j=1}^k  u_j u_j^\top \right]^{-1} w_i} \\
        &\leq \frac{w_i^\top \left[ \Lambda + \sum_{j=1}^k  u_j u_j^\top \right]^{-2} w_i}{w_i^\top \left[ \Lambda + \sum_{j=1}^k  u_j u_j^\top \right]^{-1} w_i} \\
        &= \frac{\sum_{j=1}^k (1+\alpha)^{-2} (r^{-(j-1)})^{-2} w_{i,j}^2 + \sum_{j=k+1}^d (r^{-(j-1)})^{-2} w_{i,j}^2}{\sum_{j=1}^k (1+\alpha)^{-1} (r^{-(j-1)})^{-1} w_{i,j}^2 + \sum_{j=k+1}^d (r^{-(j-1)})^{-1} w_{i,j}^2} \\
        &= \frac{\sum_{j=1}^k (1+\alpha)^{-2} r^{j-1} \frac{h}{d} + \sum_{j=k+1}^d r^{j-1} \frac{h}{d}}{\sum_{j=1}^k (1+\alpha)^{-1} \frac{h}{d} + \sum_{j=k+1}^d \frac{h}{d}} \\
        &= \frac{1}{1+\alpha} \frac{\sum_{j=1}^k r^{j-1} + (1+\alpha)^2 \sum_{j=k+1}^d r^{j-1}}{\sum_{j=1}^k 1 + \sum_{j=k+1}^d (1+\alpha)} \\
        &= \frac{1}{1+\alpha} \frac{\frac{1-r^k}{1-r} + (1+\alpha)^2 \frac{r^k-r^d}{1-r}}{k + (d-k) (1+\alpha)} \\
        &= \frac{1}{1+\alpha} \frac{1}{1-r} \frac{1 - r^d (1+\alpha)^2 + r^k (\alpha^2 + 2 \alpha)}{d(1+\alpha) - k\alpha}
    \end{align}

The condition that $f([k]) - f([k+1]) > f([k]) - f([k] \cup \{d+i\})$ is equivalent to $\left[f([k]) - f([k+1])\right] - \left[f([k]) - f([k] \cup \{d+i\})\right]$ being positive. We multiply this quantity by positive constants to get $G(k,d,r,\alpha)$ which is positive if and only if $f([k]) - f([k+1]) > f([k]) - f([k] \cup \{d+i\})$.

\begin{align}
    &G(d,\alpha,r,k)  \\
    &:= \frac{1+\alpha}{\alpha} \ln(1/r) (d(1+\alpha) - k\alpha) r^{-k} \left[f([k]) - f([k+1])\right] - \left[f([k]) - f([k] \cup \{d+i\})\right] \\
    &= \ln(1/r) (d(1+\alpha) - k\alpha) + \frac{1}{\alpha} \frac{\ln(1/r)}{1-r} \left(r^d(1+\alpha)^2 - 1\right) r^{-k} - (\alpha + 2) \frac{\ln(1/r)}{1-r}
\end{align}

For $r^d (1+\alpha)^2 \geq 1$, this function is convex in $k$. Thus, it will be positive for all $k$ if its positive for the minimizer $k^*$. Taking the derivative and setting it to $0$, we get $k^* = \frac{1}{\ln(1/r)} \ln\left(\frac{(1-r)\alpha^2}{(r^d(1+\alpha)^2 - 1) \ln(1/r)} \right)$. Thus the minimal value is at least,

\begin{align}
    g(d,\alpha,r) &:= \min_k G(d,\alpha,r,k) \\
    &= d \ln(1/r) (1+\alpha) - \alpha \ln \left(\frac{(1-r)\alpha^2}{(r^d(1+\alpha)^2 - 1) \ln(1/r)} \right) + \alpha - (\alpha + 2) \frac{\ln(1/r)}{1-r} \\ 
\end{align}

Putting it all together, if $g(d,\alpha,r)$ is positive, then $G(d,\alpha,r,k)$ is positive for all $k$, so $\left[f([k]) - f([k+1])\right] - \left[f([k]) - f([k] \cup \{d+i\})\right]$ is positive for all $k$, and thus the greedy algorithm will sequentially select the elements.
\end{proof}

We ran a Python Notebook to get some numerical results on $g$ using some example values of $d$, $\alpha$, and $r$:

%lines 391 - 411

    \begin{tcolorbox}[breakable, size=fbox, boxrule=1pt, pad at break*=1mm,colback=cellbackground, colframe=cellborder]
\prompt{In}{incolor}{1}{\boxspacing}
\begin{Verbatim}[commandchars=\\\{\}]
\PY{k+kn}{import} \PY{n+nn}{math}

\PY{k}{def} \PY{n+nf}{g}\PY{p}{(}\PY{n}{d}\PY{p}{,} \PY{n}{alpha}\PY{p}{,} \PY{n}{r}\PY{p}{)}\PY{p}{:}
    \PY{n}{subexpression} \PY{o}{=} \PY{p}{(}\PY{l+m+mi}{1}\PY{o}{\PYZhy{}}\PY{n}{r}\PY{p}{)} \PY{o}{*} \PY{n}{alpha}\PY{o}{*}\PY{o}{*}\PY{l+m+mi}{2} \PY{o}{/} \PY{p}{(}\PY{n}{r}\PY{o}{*}\PY{o}{*}\PY{n}{d} \PY{o}{*} \PY{p}{(}\PY{l+m+mi}{1}\PY{o}{+}\PY{n}{alpha}\PY{p}{)}\PY{o}{*}\PY{o}{*}\PY{l+m+mi}{2} \PY{o}{\PYZhy{}} \PY{l+m+mi}{1}\PY{p}{)} \PY{o}{/} \PY{n}{math}\PY{o}{.}\PY{n}{log}\PY{p}{(}\PY{l+m+mi}{1}\PY{o}{/}\PY{n}{r}\PY{p}{)}
    \PY{k}{return} \PY{n}{d}\PY{o}{*}\PY{n}{math}\PY{o}{.}\PY{n}{log}\PY{p}{(}\PY{l+m+mi}{1}\PY{o}{/}\PY{n}{r}\PY{p}{)}\PY{o}{*}\PY{p}{(}\PY{l+m+mi}{1}\PY{o}{+}\PY{n}{alpha}\PY{p}{)} \PY{o}{\PYZhy{}} \PY{n}{alpha}\PY{o}{*}\PY{n}{math}\PY{o}{.}\PY{n}{log}\PY{p}{(}\PY{n}{subexpression}\PY{p}{)} \PY{o}{+} \PY{n}{alpha} \PY{o}{\PYZhy{}} \PY{p}{(}\PY{n}{alpha}\PY{o}{+}\PY{l+m+mi}{2}\PY{p}{)}\PY{o}{*}\PY{n}{math}\PY{o}{.}\PY{n}{log}\PY{p}{(}\PY{l+m+mi}{1}\PY{o}{/}\PY{n}{r}\PY{p}{)}\PY{o}{/}\PY{p}{(}\PY{l+m+mi}{1}\PY{o}{\PYZhy{}}\PY{n}{r}\PY{p}{)}

\PY{k}{for} \PY{n}{d}\PY{p}{,} \PY{n}{alpha}\PY{p}{,} \PY{n}{r} \PY{o+ow}{in} \PY{p}{[}\PY{p}{(}\PY{l+m+mi}{4}\PY{p}{,}\PY{l+m+mi}{4}\PY{p}{,}\PY{n}{math}\PY{o}{.}\PY{n}{exp}\PY{p}{(}\PY{o}{\PYZhy{}}\PY{l+m+mi}{1}\PY{o}{/}\PY{l+m+mi}{4}\PY{p}{)}\PY{p}{)}\PY{p}{,}\PY{p}{(}\PY{l+m+mi}{4}\PY{p}{,}\PY{l+m+mf}{3.7}\PY{p}{,}\PY{l+m+mf}{0.75}\PY{p}{)}\PY{p}{,} \PY{p}{(}\PY{l+m+mi}{8}\PY{p}{,}\PY{l+m+mf}{2.8}\PY{p}{,}\PY{l+m+mf}{0.85}\PY{p}{)}\PY{p}{,} \PY{p}{(}\PY{l+m+mi}{16}\PY{p}{,}\PY{l+m+mf}{2.5}\PY{p}{,}\PY{l+m+mf}{0.93}\PY{p}{)}\PY{p}{,} \PY{p}{(}\PY{l+m+mi}{256}\PY{p}{,}\PY{l+m+mf}{2.25}\PY{p}{,}\PY{l+m+mf}{0.9955}\PY{p}{)}\PY{p}{]}\PY{p}{:}
    \PY{n+nb}{print}\PY{p}{(}\PY{l+s+sa}{f}\PY{l+s+s2}{\PYZdq{}}\PY{l+s+s2}{d=}\PY{l+s+si}{\PYZob{}}\PY{n}{d}\PY{l+s+si}{\PYZcb{}}\PY{l+s+s2}{,}\PY{l+s+se}{\PYZbs{}t}\PY{l+s+s2}{alpha=}\PY{l+s+si}{\PYZob{}}\PY{n}{alpha}\PY{l+s+si}{\PYZcb{}}\PY{l+s+s2}{,}\PY{l+s+se}{\PYZbs{}t}\PY{l+s+s2}{r=}\PY{l+s+si}{\PYZob{}}\PY{n+nb}{round}\PY{p}{(}\PY{n}{r}\PY{p}{,}\PY{l+m+mi}{6}\PY{p}{)}\PY{l+s+si}{\PYZcb{}}\PY{l+s+s2}{, }\PY{l+s+se}{\PYZbs{}t}\PY{l+s+s2}{g=}\PY{l+s+si}{\PYZob{}}\PY{n+nb}{round}\PY{p}{(}\PY{n}{g}\PY{p}{(}\PY{n}{d}\PY{p}{,}\PY{n}{alpha}\PY{p}{,}\PY{n}{r}\PY{p}{)}\PY{p}{,}\PY{l+m+mi}{6}\PY{p}{)}\PY{l+s+si}{\PYZcb{}}\PY{l+s+s2}{\PYZdq{}}\PY{p}{)}
\end{Verbatim}
\end{tcolorbox}

    \begin{Verbatim}[commandchars=\\\{\}]
d=4,    alpha=4,        r=0.778801,     g=0.033082
d=4,    alpha=3.7,      r=0.75,         g=0.010047
d=8,    alpha=2.8,      r=0.85,         g=0.013102
d=16,   alpha=2.5,      r=0.93,         g=0.013278
d=256,  alpha=2.25,     r=0.9955,       g=0.001105
    \end{Verbatim}

The first setting of values is used to prove Theorem~\ref{thm:lower} below. The other settings show that we can get tighter approximations (i.e. smaller $\alpha$) for larger $d$. Separately, we found that even for arbitrarily large $d$ and optimally chosen $r$, the expression can only be positive if $\alpha \geq 2.23$, so $d=256$ nearly achieves the lowest value of $\alpha$ (and thus approximation ratio). We now prove the main result,

\begin{theorem-repeat}[\ref{thm:lower}]
    For any $h>0$ and $d \geq 4$, if there exists an order $d$ Hadamard matrix, then there exists a $d$-dimensional problem with $n=2d$ vectors such that $\mils = \max(h,4)$ and for a cardinality constraint of $k=d$,

    \begin{align}
        \frac{f(S_\text{greedy})}{f(S^\star)} \geq \frac{1+h}{5}
    \end{align}
\end{theorem-repeat}
\begin{proof}
    Setting $\alpha=4$ and $r=\exp(-1/d)$, it suffices to prove that $g(d,4,\exp(-1/d)) > 0$ for all $d \geq 4$.
\begin{align}
    g(d,\exp(-1/d),4) = 5 + 4 \ln\left( 25 \exp(-1) - 1\right)-4\ln(4^2) \\-4\ln(d(1-\exp(-1/d))+4-6 \frac{1}{d(1 - \exp(-1/d))}
\end{align}

We now show that $g$ is increasing in $d$.

Let $u(d) = d(1-\exp(-1/d))$. Noting that $\exp(x) \leq 1 + x + \frac{x^2}{2}$ for $x \leq 0$,

\begin{align}
    u'(d) &= (1 - \exp(-1/d)) - d \exp(-1/d) \frac{1}{d} \\
    &= 1 - \left(1+\frac{1}{d}\right)\exp(-1/d) \\
    &\geq 1 - \left(1+\frac{1}{d}\right)\left(1-\frac{1}{d} + \frac{1}{2d^2}\right) \\
    &= \frac{1}{2d^2} \left(1 - \frac{1}{d}\right) > 0
\end{align} 

So $u$ is increasing in $d$. Furthermore, noting that $\exp(x) \geq 1+x$ for all $x$, we can show $u(d) \leq 1$,

\begin{align}
    u(d) &= d(1-\exp(-1/d)) \\
    &\leq d(1 - (1-1/d)) \\
    &= 1
\end{align}

Since $4\ln(1/u) - 6/u$ is increasing in $u$ for $u \in (0,1)$, $g(d,\exp(-1/d),4)$ is increasing in $d$.

Thus

\begin{align}
    g(d,\exp(-1/d),4) &\geq g(4,\exp(-1/4),4) \\
    &> 0.03
\end{align}

The last equation is from the numerical result.
\end{proof}

\section{Lower Bound Example Code}
\label{app:lower-bound-experiment}

Below is the code to numerically check the explicit example with $d=4$.

% lines 413 - 502

    \begin{tcolorbox}[breakable, size=fbox, boxrule=1pt, pad at break*=1mm,colback=cellbackground, colframe=cellborder]
\prompt{In}{incolor}{2}{\boxspacing}
\begin{Verbatim}[commandchars=\\\{\}]
\PY{k+kn}{import} \PY{n+nn}{numpy} \PY{k}{as} \PY{n+nn}{np}
\PY{k+kn}{import} \PY{n+nn}{matplotlib}\PY{n+nn}{.}\PY{n+nn}{pyplot} \PY{k}{as} \PY{n+nn}{plt}

\PY{n}{H} \PY{o}{=} \PY{n}{np}\PY{o}{.}\PY{n}{array}\PY{p}{(}\PY{p}{[}\PY{p}{[} \PY{l+m+mi}{1}\PY{p}{,}  \PY{l+m+mi}{1}\PY{p}{,}  \PY{l+m+mi}{1}\PY{p}{,}  \PY{l+m+mi}{1}\PY{p}{]}\PY{p}{,}
              \PY{p}{[} \PY{l+m+mi}{1}\PY{p}{,} \PY{o}{\PYZhy{}}\PY{l+m+mi}{1}\PY{p}{,}  \PY{l+m+mi}{1}\PY{p}{,} \PY{o}{\PYZhy{}}\PY{l+m+mi}{1}\PY{p}{]}\PY{p}{,}
              \PY{p}{[} \PY{l+m+mi}{1}\PY{p}{,}  \PY{l+m+mi}{1}\PY{p}{,} \PY{o}{\PYZhy{}}\PY{l+m+mi}{1}\PY{p}{,} \PY{o}{\PYZhy{}}\PY{l+m+mi}{1}\PY{p}{]}\PY{p}{,}
              \PY{p}{[} \PY{l+m+mi}{1}\PY{p}{,} \PY{o}{\PYZhy{}}\PY{l+m+mi}{1}\PY{p}{,} \PY{o}{\PYZhy{}}\PY{l+m+mi}{1}\PY{p}{,}  \PY{l+m+mi}{1}\PY{p}{]}\PY{p}{]}\PY{p}{)}

\PY{k}{class} \PY{n+nc}{Problem}\PY{p}{:}
    \PY{k}{def} \PY{n+nf}{mils}\PY{p}{(}\PY{n+nb+bp}{self}\PY{p}{)}\PY{p}{:}
        \PY{k}{return} \PY{n+nb}{max}\PY{p}{(}\PY{p}{[}\PY{n}{v} \PY{o}{@} \PY{n}{np}\PY{o}{.}\PY{n}{linalg}\PY{o}{.}\PY{n}{inv}\PY{p}{(}\PY{n+nb+bp}{self}\PY{o}{.}\PY{n}{Lam}\PY{p}{)} \PY{o}{@} \PY{n}{v} \PY{k}{for} \PY{n}{v} \PY{o+ow}{in} \PY{n+nb+bp}{self}\PY{o}{.}\PY{n}{vectors}\PY{p}{]}\PY{p}{)}
             
    \PY{k}{def} \PY{n+nf}{f}\PY{p}{(}\PY{n+nb+bp}{self}\PY{p}{,} \PY{n}{S}\PY{p}{)}\PY{p}{:}
        \PY{n}{cov} \PY{o}{=} \PY{n}{np}\PY{o}{.}\PY{n}{array}\PY{p}{(}\PY{n+nb+bp}{self}\PY{o}{.}\PY{n}{Lam}\PY{p}{)}
        \PY{k}{for} \PY{n}{i} \PY{o+ow}{in} \PY{n}{S}\PY{p}{:}
             \PY{n}{cov} \PY{o}{+}\PY{o}{=} \PY{n}{np}\PY{o}{.}\PY{n}{outer}\PY{p}{(}\PY{n+nb+bp}{self}\PY{o}{.}\PY{n}{vectors}\PY{p}{[}\PY{n}{i}\PY{p}{]}\PY{p}{,} \PY{n+nb+bp}{self}\PY{o}{.}\PY{n}{vectors}\PY{p}{[}\PY{n}{i}\PY{p}{]}\PY{p}{)}
        \PY{k}{return} \PY{n}{np}\PY{o}{.}\PY{n}{trace}\PY{p}{(}\PY{n}{np}\PY{o}{.}\PY{n}{linalg}\PY{o}{.}\PY{n}{inv}\PY{p}{(}\PY{n}{cov}\PY{p}{)}\PY{p}{)}
             
\PY{k}{class} \PY{n+nc}{ConstructedProblem}\PY{p}{(}\PY{n}{Problem}\PY{p}{)}\PY{p}{:}
    \PY{k}{def} \PY{n+nf+fm}{\PYZus{}\PYZus{}init\PYZus{}\PYZus{}}\PY{p}{(}\PY{n+nb+bp}{self}\PY{p}{,} \PY{n}{h}\PY{p}{)}\PY{p}{:}
        \PY{n+nb+bp}{self}\PY{o}{.}\PY{n}{d} \PY{o}{=} \PY{l+m+mi}{4}
        \PY{n+nb+bp}{self}\PY{o}{.}\PY{n}{n} \PY{o}{=} \PY{l+m+mi}{4}
        \PY{n+nb+bp}{self}\PY{o}{.}\PY{n}{Lam} \PY{o}{=} \PY{n}{np}\PY{o}{.}\PY{n}{diag}\PY{p}{(}\PY{p}{[}\PY{n}{math}\PY{o}{.}\PY{n}{exp}\PY{p}{(}\PY{n}{i}\PY{o}{/}\PY{l+m+mi}{4}\PY{p}{)} \PY{k}{for} \PY{n}{i} \PY{o+ow}{in} \PY{n+nb}{range}\PY{p}{(}\PY{l+m+mi}{4}\PY{p}{)}\PY{p}{]}\PY{p}{)}
           
        \PY{n+nb+bp}{self}\PY{o}{.}\PY{n}{vectors} \PY{o}{=} \PY{p}{[}\PY{p}{]}
        \PY{n}{scaling} \PY{o}{=} \PY{n}{np}\PY{o}{.}\PY{n}{diag}\PY{p}{(}\PY{p}{[}\PY{n}{math}\PY{o}{.}\PY{n}{exp}\PY{p}{(}\PY{n}{i}\PY{o}{/}\PY{l+m+mi}{8}\PY{p}{)} \PY{k}{for} \PY{n}{i} \PY{o+ow}{in} \PY{n+nb}{range}\PY{p}{(}\PY{l+m+mi}{4}\PY{p}{)}\PY{p}{]}\PY{p}{)} 
        \PY{k}{for} \PY{n}{i} \PY{o+ow}{in} \PY{n+nb}{range}\PY{p}{(}\PY{l+m+mi}{4}\PY{p}{)}\PY{p}{:}
            \PY{n+nb+bp}{self}\PY{o}{.}\PY{n}{vectors}\PY{o}{.}\PY{n}{append}\PY{p}{(}\PY{l+m+mi}{2} \PY{o}{*} \PY{n}{scaling} \PY{o}{@} \PY{n}{np}\PY{o}{.}\PY{n}{eye}\PY{p}{(}\PY{l+m+mi}{4}\PY{p}{)}\PY{p}{[}\PY{n}{i}\PY{p}{]}\PY{p}{)}
        \PY{k}{for} \PY{n}{i} \PY{o+ow}{in} \PY{n+nb}{range}\PY{p}{(}\PY{l+m+mi}{4}\PY{p}{)}\PY{p}{:}
            \PY{n+nb+bp}{self}\PY{o}{.}\PY{n}{vectors}\PY{o}{.}\PY{n}{append}\PY{p}{(}\PY{n}{math}\PY{o}{.}\PY{n}{sqrt}\PY{p}{(}\PY{n}{h}\PY{p}{)}\PY{o}{/}\PY{l+m+mi}{2} \PY{o}{*} \PY{n}{scaling} \PY{o}{@}  \PY{n}{H}\PY{p}{[}\PY{n}{i}\PY{p}{]}\PY{p}{)}
                
\PY{k}{def} \PY{n+nf}{Greedy}\PY{p}{(}\PY{n}{problem}\PY{p}{,}\PY{n}{k}\PY{p}{)}\PY{p}{:}
    \PY{n}{S} \PY{o}{=} \PY{p}{[}\PY{p}{]}
    \PY{k}{for} \PY{n}{t} \PY{o+ow}{in} \PY{n+nb}{range}\PY{p}{(}\PY{n}{k}\PY{p}{)}\PY{p}{:}
        \PY{n}{greedy\PYZus{}next} \PY{o}{=} \PY{n}{np}\PY{o}{.}\PY{n}{argmin}\PY{p}{(}\PY{p}{[}\PY{n}{problem}\PY{o}{.}\PY{n}{f}\PY{p}{(}\PY{n}{S}\PY{o}{+}\PY{p}{[}\PY{n}{i}\PY{p}{]}\PY{p}{)} \PY{k}{if} \PY{n}{i} \PY{o+ow}{not} \PY{o+ow}{in} \PY{n}{S} \PY{k}{else} \PY{n}{np}\PY{o}{.}\PY{n}{inf} \PY{k}{for} \PY{n}{i} \PY{o+ow}{in} \PY{n+nb}{range}\PY{p}{(}\PY{n}{problem}\PY{o}{.}\PY{n}{n}\PY{p}{)}\PY{p}{]}\PY{p}{)}
        \PY{n}{S}\PY{o}{.}\PY{n}{append}\PY{p}{(}\PY{n}{greedy\PYZus{}next}\PY{p}{)}
    \PY{k}{return} \PY{n}{S}
\end{Verbatim}
\end{tcolorbox}

    \begin{tcolorbox}[breakable, size=fbox, boxrule=1pt, pad at break*=1mm,colback=cellbackground, colframe=cellborder]
\prompt{In}{incolor}{3}{\boxspacing}
\begin{Verbatim}[commandchars=\\\{\}]
\PY{k}{for} \PY{n}{h} \PY{o+ow}{in} \PY{p}{[}\PY{l+m+mi}{10}\PY{p}{,}\PY{l+m+mi}{100}\PY{p}{,}\PY{l+m+mi}{1000}\PY{p}{]}\PY{p}{:}
    \PY{n}{prob} \PY{o}{=} \PY{n}{ConstructedProblem}\PY{p}{(}\PY{n}{h}\PY{p}{)}
    
    \PY{n}{S}\PY{o}{=}\PY{n}{Greedy}\PY{p}{(}\PY{n}{prob}\PY{p}{,}\PY{l+m+mi}{4}\PY{p}{)}
    \PY{n}{greedy\PYZus{}value} \PY{o}{=} \PY{n}{prob}\PY{o}{.}\PY{n}{f}\PY{p}{(}\PY{n}{S}\PY{p}{)}
    \PY{n}{better\PYZus{}value} \PY{o}{=} \PY{n}{prob}\PY{o}{.}\PY{n}{f}\PY{p}{(}\PY{p}{[}\PY{l+m+mi}{4}\PY{p}{,}\PY{l+m+mi}{5}\PY{p}{,}\PY{l+m+mi}{6}\PY{p}{,}\PY{l+m+mi}{7}\PY{p}{]}\PY{p}{)}

    \PY{n+nb}{print}\PY{p}{(}\PY{l+m+mi}{32}\PY{o}{*}\PY{l+s+s2}{\PYZdq{}}\PY{l+s+s2}{\PYZhy{}}\PY{l+s+s2}{\PYZdq{}}\PY{p}{)}
    \PY{n+nb}{print}\PY{p}{(}\PY{l+s+sa}{f}\PY{l+s+s2}{\PYZdq{}}\PY{l+s+s2}{h=}\PY{l+s+si}{\PYZob{}}\PY{n}{h}\PY{l+s+si}{\PYZcb{}}\PY{l+s+s2}{\PYZdq{}}\PY{p}{)}
    \PY{n+nb}{print}\PY{p}{(}\PY{l+s+sa}{f}\PY{l+s+s2}{\PYZdq{}}\PY{l+s+s2}{MILS=}\PY{l+s+si}{\PYZob{}}\PY{n}{prob}\PY{o}{.}\PY{n}{mils}\PY{p}{(}\PY{p}{)}\PY{l+s+si}{\PYZcb{}}\PY{l+s+s2}{\PYZdq{}}\PY{p}{)}
    \PY{n+nb}{print}\PY{p}{(}\PY{l+s+sa}{f}\PY{l+s+s2}{\PYZdq{}}\PY{l+s+s2}{Greedy Set: }\PY{l+s+si}{\PYZob{}}\PY{n}{S}\PY{l+s+si}{\PYZcb{}}\PY{l+s+s2}{\PYZdq{}}\PY{p}{)}
    \PY{n+nb}{print}\PY{p}{(}\PY{l+s+sa}{f}\PY{l+s+s2}{\PYZdq{}}\PY{l+s+s2}{Greedy Value: }\PY{l+s+si}{\PYZob{}}\PY{n}{greedy\PYZus{}value}\PY{l+s+si}{\PYZcb{}}\PY{l+s+s2}{\PYZdq{}}\PY{p}{)}
    \PY{n+nb}{print}\PY{p}{(}\PY{l+s+sa}{f}\PY{l+s+s2}{\PYZdq{}}\PY{l+s+s2}{Better Value: }\PY{l+s+si}{\PYZob{}}\PY{n}{better\PYZus{}value}\PY{l+s+si}{\PYZcb{}}\PY{l+s+s2}{\PYZdq{}}\PY{p}{)}
    \PY{n+nb}{print}\PY{p}{(}\PY{l+s+sa}{f}\PY{l+s+s2}{\PYZdq{}}\PY{l+s+s2}{Approx Ratio: }\PY{l+s+si}{\PYZob{}}\PY{n}{greedy\PYZus{}value}\PY{o}{/}\PY{n}{better\PYZus{}value}\PY{l+s+si}{\PYZcb{}}\PY{l+s+s2}{\PYZdq{}}\PY{p}{)}
    \PY{n+nb}{print}\PY{p}{(}\PY{p}{)}
\end{Verbatim}
\end{tcolorbox}

    \begin{Verbatim}[commandchars=\\\{\}]
--------------------------------
h=10
MILS=10.0
Greedy Set: [0, 1, 2, 3]
Greedy Value: 0.5715395991050106
Better Value: 0.25979072686591387
Approx Ratio: 2.2000000000000006

--------------------------------
h=100
MILS=100.0
Greedy Set: [0, 1, 2, 3]
Greedy Value: 0.5715395991050106
Better Value: 0.028294039559653993
Approx Ratio: 20.2

--------------------------------
h=1000
MILS=1000.0
Greedy Set: [0, 1, 2, 3]
Greedy Value: 0.5715395991050106
Better Value: 0.0028548431523726806
Approx Ratio: 200.2

    \end{Verbatim}

\section{Illustrative Example Code}
\label{app:illustrative-example}

% lines 504 - 591

    \begin{tcolorbox}[breakable, size=fbox, boxrule=1pt, pad at break*=1mm,colback=cellbackground, colframe=cellborder]
\prompt{In}{incolor}{4}{\boxspacing}
\begin{Verbatim}[commandchars=\\\{\}]
\PY{k}{class} \PY{n+nc}{RandomProblem}\PY{p}{(}\PY{n}{Problem}\PY{p}{)}\PY{p}{:}
    \PY{k}{def} \PY{n+nf+fm}{\PYZus{}\PYZus{}init\PYZus{}\PYZus{}}\PY{p}{(}\PY{n+nb+bp}{self}\PY{p}{,} \PY{n}{d}\PY{p}{,} \PY{n}{n}\PY{p}{)}\PY{p}{:}
        \PY{n+nb+bp}{self}\PY{o}{.}\PY{n}{d} \PY{o}{=} \PY{n}{d}
        \PY{n+nb+bp}{self}\PY{o}{.}\PY{n}{n} \PY{o}{=} \PY{n}{n}
        \PY{n+nb+bp}{self}\PY{o}{.}\PY{n}{Lam} \PY{o}{=} \PY{n}{np}\PY{o}{.}\PY{n}{eye}\PY{p}{(}\PY{n}{d}\PY{p}{)}
             
        \PY{n+nb+bp}{self}\PY{o}{.}\PY{n}{vectors} \PY{o}{=} \PY{p}{[}\PY{p}{]}
        \PY{n}{rng} \PY{o}{=} \PY{n}{np}\PY{o}{.}\PY{n}{random}\PY{o}{.}\PY{n}{default\PYZus{}rng}\PY{p}{(}\PY{n}{seed}\PY{o}{=}\PY{l+m+mi}{0}\PY{p}{)}
        \PY{k}{for} \PY{n}{i} \PY{o+ow}{in} \PY{n+nb}{range}\PY{p}{(}\PY{n}{n}\PY{p}{)}\PY{p}{:}
            \PY{n}{v} \PY{o}{=} \PY{n}{rng}\PY{o}{.}\PY{n}{normal}\PY{p}{(}\PY{n}{size}\PY{o}{=}\PY{p}{(}\PY{n}{d}\PY{p}{,}\PY{p}{)}\PY{p}{)}
            \PY{n}{v} \PY{o}{/}\PY{o}{=} \PY{n}{np}\PY{o}{.}\PY{n}{linalg}\PY{o}{.}\PY{n}{norm}\PY{p}{(}\PY{n}{v}\PY{p}{)}
            \PY{n+nb+bp}{self}\PY{o}{.}\PY{n}{vectors}\PY{o}{.}\PY{n}{append}\PY{p}{(}\PY{n}{v}\PY{p}{)}
            
\PY{k}{def} \PY{n+nf}{Chamon\PYZus{}alpha}\PY{p}{(}\PY{n}{a}\PY{p}{,}\PY{n}{b}\PY{p}{)}\PY{p}{:}
    \PY{k}{return} \PY{l+m+mi}{1}\PY{o}{/}\PY{p}{(}\PY{l+m+mi}{1}\PY{o}{+}\PY{n}{a}\PY{p}{)}

\PY{k}{def} \PY{n+nf}{Chamon\PYZus{}reduction\PYZus{}approx\PYZus{}ratio}\PY{p}{(}\PY{n}{k}\PY{p}{)}\PY{p}{:}
    \PY{k}{if} \PY{n}{k}\PY{o}{==}\PY{l+m+mi}{0}\PY{p}{:}
        \PY{k}{return} \PY{l+m+mi}{1}
    \PY{n}{product\PYZus{}terms} \PY{o}{=} \PY{p}{[}\PY{l+m+mi}{1} \PY{o}{\PYZhy{}} \PY{l+m+mi}{1}\PY{o}{/}\PY{n}{np}\PY{o}{.}\PY{n}{sum}\PY{p}{(}\PY{p}{[}\PY{l+m+mi}{1}\PY{o}{/}\PY{n}{Chamon\PYZus{}alpha}\PY{p}{(}\PY{n}{h}\PY{p}{,}\PY{n}{h}\PY{o}{+}\PY{n}{s}\PY{p}{)} \PY{k}{for} \PY{n}{s} \PY{o+ow}{in} \PY{n+nb}{range}\PY{p}{(}\PY{n}{k}\PY{p}{)}\PY{p}{]}\PY{p}{)} \PY{k}{for} \PY{n}{h} \PY{o+ow}{in} \PY{n+nb}{range}\PY{p}{(}\PY{n}{k}\PY{p}{)}\PY{p}{]}
    \PY{k}{return} \PY{l+m+mi}{1} \PY{o}{\PYZhy{}} \PY{n}{np}\PY{o}{.}\PY{n}{product}\PY{p}{(}\PY{n}{product\PYZus{}terms}\PY{p}{)}

\PY{k}{def} \PY{n+nf}{Bian\PYZus{}gamma\PYZus{}alpha}\PY{p}{(}\PY{n}{vectors}\PY{p}{)}\PY{p}{:}
    \PY{n}{X\PYZus{}norm\PYZus{}squared} \PY{o}{=} \PY{n}{np}\PY{o}{.}\PY{n}{linalg}\PY{o}{.}\PY{n}{norm}\PY{p}{(}\PY{n}{np}\PY{o}{.}\PY{n}{array}\PY{p}{(}\PY{n}{vectors}\PY{p}{)}\PY{p}{,}\PY{n+nb}{ord}\PY{o}{=}\PY{l+m+mi}{2}\PY{p}{)}\PY{o}{*}\PY{o}{*}\PY{l+m+mi}{2}
    \PY{k}{return} \PY{p}{\PYZob{}}\PY{l+s+s2}{\PYZdq{}}\PY{l+s+s2}{gamma}\PY{l+s+s2}{\PYZdq{}}\PY{p}{:} \PY{l+m+mi}{1}\PY{o}{/}\PY{n}{X\PYZus{}norm\PYZus{}squared}\PY{o}{/}\PY{p}{(}\PY{l+m+mi}{1}\PY{o}{+}\PY{n}{X\PYZus{}norm\PYZus{}squared}\PY{p}{)}\PY{p}{,} \PY{l+s+s2}{\PYZdq{}}\PY{l+s+s2}{alpha}\PY{l+s+s2}{\PYZdq{}}\PY{p}{:} \PY{l+m+mi}{1} \PY{o}{\PYZhy{}} \PY{l+m+mi}{1}\PY{o}{/}\PY{n}{X\PYZus{}norm\PYZus{}squared}\PY{o}{/}\PY{p}{(}\PY{l+m+mi}{1}\PY{o}{+}\PY{n}{X\PYZus{}norm\PYZus{}squared}\PY{p}{)}\PY{p}{\PYZcb{}}

\PY{k}{def} \PY{n+nf}{Bian\PYZus{}reduction\PYZus{}approx\PYZus{}ratio}\PY{p}{(}\PY{n}{alpha}\PY{p}{,}\PY{n}{gamma}\PY{p}{)}\PY{p}{:}
    \PY{k}{return} \PY{l+m+mi}{1}\PY{o}{/}\PY{n}{alpha} \PY{o}{*} \PY{p}{(}\PY{l+m+mi}{1} \PY{o}{\PYZhy{}} \PY{p}{(}\PY{p}{(}\PY{n}{k} \PY{o}{\PYZhy{}} \PY{n}{alpha}\PY{o}{*}\PY{n}{gamma}\PY{p}{)}\PY{o}{/}\PY{n}{k}\PY{p}{)}\PY{o}{*}\PY{o}{*}\PY{n}{k}\PY{p}{)}

\PY{k}{def} \PY{n+nf}{reduction\PYZus{}approx\PYZus{}ratio\PYZus{}lb}\PY{p}{(}\PY{n}{greedy\PYZus{}value}\PY{p}{,}\PY{n}{f\PYZus{}0}\PY{p}{,}\PY{n}{reduction\PYZus{}approx\PYZus{}ratio}\PY{p}{)}\PY{p}{:}
    \PY{k}{return} \PY{n}{f\PYZus{}0} \PY{o}{\PYZhy{}} \PY{p}{(}\PY{n}{f\PYZus{}0} \PY{o}{\PYZhy{}} \PY{n}{greedy\PYZus{}value}\PY{p}{)}\PY{o}{/}\PY{n}{reduction\PYZus{}approx\PYZus{}ratio}
\end{Verbatim}
\end{tcolorbox}

    \begin{tcolorbox}[breakable, size=fbox, boxrule=1pt, pad at break*=1mm,colback=cellbackground, colframe=cellborder]
\prompt{In}{incolor}{5}{\boxspacing}
\begin{Verbatim}[commandchars=\\\{\}]
\PY{n}{d}\PY{p}{,} \PY{n}{n}\PY{p}{,} \PY{n}{k} \PY{o}{=} \PY{l+m+mi}{20}\PY{p}{,} \PY{l+m+mi}{1000}\PY{p}{,} \PY{l+m+mi}{100}

\PY{n}{prob} \PY{o}{=} \PY{n}{RandomProblem}\PY{p}{(}\PY{n}{d}\PY{p}{,}\PY{n}{n}\PY{p}{)}
\PY{n}{f\PYZus{}0} \PY{o}{=} \PY{n}{prob}\PY{o}{.}\PY{n}{f}\PY{p}{(}\PY{p}{[}\PY{p}{]}\PY{p}{)}
\PY{n}{our\PYZus{}approx\PYZus{}ratio} \PY{o}{=} \PY{p}{(}\PY{n}{prob}\PY{o}{.}\PY{n}{mils}\PY{p}{(}\PY{p}{)} \PY{o}{+} \PY{l+m+mi}{1}\PY{o}{/}\PY{p}{(}\PY{l+m+mi}{1}\PY{o}{\PYZhy{}}\PY{l+m+mi}{1}\PY{o}{/}\PY{n}{math}\PY{o}{.}\PY{n}{e}\PY{p}{)}\PY{p}{)}

\PY{n}{S} \PY{o}{=} \PY{n}{Greedy}\PY{p}{(}\PY{n}{prob}\PY{p}{,}\PY{n}{k}\PY{p}{)}
\PY{n}{greedy\PYZus{}values} \PY{o}{=} \PY{n}{np}\PY{o}{.}\PY{n}{array}\PY{p}{(}\PY{p}{[}\PY{n}{prob}\PY{o}{.}\PY{n}{f}\PY{p}{(}\PY{n}{S}\PY{p}{[}\PY{p}{:}\PY{n}{t}\PY{p}{]}\PY{p}{)} \PY{k}{for} \PY{n}{t} \PY{o+ow}{in} \PY{n+nb}{range}\PY{p}{(}\PY{n}{k}\PY{o}{+}\PY{l+m+mi}{1}\PY{p}{)}\PY{p}{]}\PY{p}{)}

\PY{n+nb}{print}\PY{p}{(}\PY{l+m+mi}{32}\PY{o}{*}\PY{l+s+s2}{\PYZdq{}}\PY{l+s+s2}{\PYZhy{}}\PY{l+s+s2}{\PYZdq{}}\PY{p}{)}
\PY{n+nb}{print}\PY{p}{(}\PY{l+s+s2}{\PYZdq{}}\PY{l+s+s2}{Bian Analysis at k=1}\PY{l+s+s2}{\PYZdq{}}\PY{p}{)}
\PY{n+nb}{print}\PY{p}{(}\PY{n}{Bian\PYZus{}gamma\PYZus{}alpha}\PY{p}{(}\PY{n}{prob}\PY{o}{.}\PY{n}{vectors}\PY{p}{)}\PY{p}{)}
\PY{n+nb}{print}\PY{p}{(}\PY{l+s+s2}{\PYZdq{}}\PY{l+s+s2}{Bian Reduction Approximation Ratio: }\PY{l+s+si}{\PYZob{}\PYZcb{}}\PY{l+s+s2}{\PYZdq{}}\PY{o}{.}\PY{n}{format}\PY{p}{(}\PY{n+nb}{round}\PY{p}{(}\PY{n}{Bian\PYZus{}reduction\PYZus{}approx\PYZus{}ratio}\PY{p}{(}\PY{o}{*}\PY{o}{*}\PY{n}{Bian\PYZus{}gamma\PYZus{}alpha}\PY{p}{(}\PY{n}{prob}\PY{o}{.}\PY{n}{vectors}\PY{p}{)}\PY{p}{)}\PY{p}{,}\PY{l+m+mi}{6}\PY{p}{)}\PY{p}{)}\PY{p}{)}
\PY{n+nb}{print}\PY{p}{(}\PY{l+s+s2}{\PYZdq{}}\PY{l+s+s2}{Bian Optimal Lower Bound: }\PY{l+s+si}{\PYZob{}\PYZcb{}}\PY{l+s+s2}{\PYZdq{}}\PY{o}{.}\PY{n}{format}\PY{p}{(}\PY{n+nb}{round}\PY{p}{(}\PY{n}{reduction\PYZus{}approx\PYZus{}ratio\PYZus{}lb}\PY{p}{(}\PY{n}{greedy\PYZus{}values}\PY{p}{[}\PY{l+m+mi}{1}\PY{p}{]}\PY{p}{,}\PY{n}{f\PYZus{}0}\PY{p}{,}\PY{n}{Bian\PYZus{}reduction\PYZus{}approx\PYZus{}ratio}\PY{p}{(}\PY{o}{*}\PY{o}{*}\PY{n}{Bian\PYZus{}gamma\PYZus{}alpha}\PY{p}{(}\PY{n}{prob}\PY{o}{.}\PY{n}{vectors}\PY{p}{)}\PY{p}{)}\PY{p}{)}\PY{p}{,}\PY{l+m+mi}{6}\PY{p}{)}\PY{p}{)}\PY{p}{)}
\PY{n+nb}{print}\PY{p}{(}\PY{l+m+mi}{32}\PY{o}{*}\PY{l+s+s2}{\PYZdq{}}\PY{l+s+s2}{\PYZhy{}}\PY{l+s+s2}{\PYZdq{}}\PY{p}{)}

\PY{n}{our\PYZus{}lb\PYZus{}values} \PY{o}{=} \PY{n}{greedy\PYZus{}values}\PY{o}{/}\PY{n}{our\PYZus{}approx\PYZus{}ratio}
\PY{n}{Chamon\PYZus{}lb\PYZus{}values} \PY{o}{=} \PY{n}{np}\PY{o}{.}\PY{n}{array}\PY{p}{(}\PY{p}{[}\PY{n}{reduction\PYZus{}approx\PYZus{}ratio\PYZus{}lb}\PY{p}{(}\PY{n}{greedy\PYZus{}values}\PY{p}{[}\PY{n}{t}\PY{p}{]}\PY{p}{,}\PY{n}{f\PYZus{}0}\PY{p}{,}\PY{n}{Chamon\PYZus{}reduction\PYZus{}approx\PYZus{}ratio}\PY{p}{(}\PY{n}{t}\PY{p}{)}\PY{p}{)} \PY{k}{for} \PY{n}{t} \PY{o+ow}{in} \PY{n+nb}{range}\PY{p}{(}\PY{n}{k}\PY{o}{+}\PY{l+m+mi}{1}\PY{p}{)}\PY{p}{]}\PY{p}{)}

\PY{n}{plt}\PY{o}{.}\PY{n}{subplots}\PY{p}{(}\PY{n}{figsize}\PY{o}{=}\PY{p}{(}\PY{l+m+mi}{8}\PY{p}{,} \PY{l+m+mi}{6}\PY{p}{)}\PY{p}{)}
\PY{n}{plt}\PY{o}{.}\PY{n}{plot}\PY{p}{(}\PY{n+nb}{range}\PY{p}{(}\PY{n}{k}\PY{o}{+}\PY{l+m+mi}{1}\PY{p}{)}\PY{p}{,}\PY{n}{greedy\PYZus{}values}\PY{p}{,} \PY{n}{color}\PY{o}{=}\PY{l+s+s1}{\PYZsq{}}\PY{l+s+s1}{blue}\PY{l+s+s1}{\PYZsq{}}\PY{p}{,} \PY{n}{linewidth}\PY{o}{=}\PY{l+m+mi}{2}\PY{p}{,} \PY{n}{label}\PY{o}{=}\PY{l+s+s1}{\PYZsq{}}\PY{l+s+s1}{Greedy Algorithm}\PY{l+s+s1}{\PYZsq{}}\PY{p}{)}
\PY{n}{plt}\PY{o}{.}\PY{n}{plot}\PY{p}{(}\PY{n+nb}{range}\PY{p}{(}\PY{n}{k}\PY{o}{+}\PY{l+m+mi}{1}\PY{p}{)}\PY{p}{,}\PY{n}{our\PYZus{}lb\PYZus{}values}\PY{p}{,} \PY{n}{color}\PY{o}{=}\PY{l+s+s1}{\PYZsq{}}\PY{l+s+s1}{green}\PY{l+s+s1}{\PYZsq{}}\PY{p}{,} \PY{n}{linewidth}\PY{o}{=}\PY{l+m+mi}{2}\PY{p}{,} \PY{n}{label}\PY{o}{=}\PY{l+s+s1}{\PYZsq{}}\PY{l+s+s1}{Our Lower Bound on Optimal}\PY{l+s+s1}{\PYZsq{}}\PY{p}{,} \PY{n}{linestyle}\PY{o}{=}\PY{l+s+s1}{\PYZsq{}}\PY{l+s+s1}{dashed}\PY{l+s+s1}{\PYZsq{}}\PY{p}{)}
\PY{n}{plt}\PY{o}{.}\PY{n}{plot}\PY{p}{(}\PY{n+nb}{range}\PY{p}{(}\PY{n}{k}\PY{o}{+}\PY{l+m+mi}{1}\PY{p}{)}\PY{p}{,}\PY{n}{Chamon\PYZus{}lb\PYZus{}values}\PY{p}{,} \PY{n}{color}\PY{o}{=}\PY{l+s+s1}{\PYZsq{}}\PY{l+s+s1}{red}\PY{l+s+s1}{\PYZsq{}}\PY{p}{,} \PY{n}{linewidth}\PY{o}{=}\PY{l+m+mi}{2}\PY{p}{,} \PY{n}{label}\PY{o}{=}\PY{l+s+s1}{\PYZsq{}}\PY{l+s+s1}{Chamon}\PY{l+s+se}{\PYZbs{}\PYZsq{}}\PY{l+s+s1}{s Lower Bound on Optimal}\PY{l+s+s1}{\PYZsq{}}\PY{p}{,} \PY{n}{linestyle}\PY{o}{=}\PY{l+s+s1}{\PYZsq{}}\PY{l+s+s1}{dotted}\PY{l+s+s1}{\PYZsq{}}\PY{p}{)}

\PY{n}{plt}\PY{o}{.}\PY{n}{xlabel}\PY{p}{(}\PY{l+s+s1}{\PYZsq{}}\PY{l+s+s1}{k}\PY{l+s+s1}{\PYZsq{}}\PY{p}{,} \PY{n}{fontsize}\PY{o}{=}\PY{l+m+mi}{13}\PY{p}{)}
\PY{n}{plt}\PY{o}{.}\PY{n}{ylabel}\PY{p}{(}\PY{l+s+s1}{\PYZsq{}}\PY{l+s+s1}{f value}\PY{l+s+s1}{\PYZsq{}}\PY{p}{,} \PY{n}{fontsize}\PY{o}{=}\PY{l+m+mi}{13}\PY{p}{)}
\PY{n}{plt}\PY{o}{.}\PY{n}{xlim}\PY{p}{(}\PY{l+m+mi}{0}\PY{p}{,} \PY{n}{k}\PY{p}{)}
\PY{n}{plt}\PY{o}{.}\PY{n}{ylim}\PY{p}{(}\PY{l+m+mi}{0}\PY{p}{,} \PY{n}{d}\PY{p}{)}
\PY{n}{plt}\PY{o}{.}\PY{n}{grid}\PY{p}{(}\PY{k+kc}{True}\PY{p}{,} \PY{n}{linestyle}\PY{o}{=}\PY{l+s+s1}{\PYZsq{}}\PY{l+s+s1}{\PYZhy{}\PYZhy{}}\PY{l+s+s1}{\PYZsq{}}\PY{p}{,} \PY{n}{alpha}\PY{o}{=}\PY{l+m+mf}{0.7}\PY{p}{)}
\PY{n}{plt}\PY{o}{.}\PY{n}{legend}\PY{p}{(}\PY{p}{)}
\PY{n}{plt}\PY{o}{.}\PY{n}{tight\PYZus{}layout}\PY{p}{(}\PY{p}{)}
\PY{n}{plt}\PY{o}{.}\PY{n}{show}\PY{p}{(}\PY{p}{)}
\end{Verbatim}
\end{tcolorbox}

    \begin{Verbatim}[commandchars=\\\{\}]
--------------------------------
Bian Analysis at k=1
\{'gamma': 0.0002594098407928644, 'alpha': 0.9997405901592071\}
Bian Reduction Approximation Ratio: 0.000259
Bian Optimal Lower Bound: -1907.699383
--------------------------------
    \end{Verbatim}

    \begin{center}
    \adjustimage{max size={0.9\linewidth}{0.9\paperheight}}{Numerical_Simulations/output_4_1.png}
    \end{center}
    { \hspace*{\fill} \\}

%%%%%%%%%%%%%%%%%%%%%%%%%%%%%%%%%%%%%%%%%%%%%%%%%%%%%%%%%%%%

% \newpage
% \input{checklist.tex}

\end{document}